\documentclass{article}
\usepackage{microtype}
\usepackage{comment}
\usepackage{graphicx}
\usepackage{multirow}
\usepackage{subfigure}
\usepackage{booktabs} 
\usepackage[dvipsnames]{xcolor}
\usepackage{amsmath}
\usepackage{makecell}
\usepackage{amsfonts}
\usepackage{wrapfig}
\usepackage{amssymb}
\usepackage{colortbl}
\usepackage{comment}
\usepackage{multicol}
\definecolor{mygray}{gray}{.8}
\newcommand{\yplu}[1]{\textcolor{red}{(2prime: #1)}}

\newcommand{\zi}[1]{\textcolor{ForestGreen}{(Zach: #1)}}

\newcommand{\blockcomment}[1]{}
\usepackage{amsthm}
\newtheorem{theorem}{Theorem}
\newtheorem{remark}{Remark}

\newtheorem{lemma}{Lemma}
\newtheorem{assumption}{Assumption}

\def\x{{\boldsymbol{x}}}
\def\h{{\boldsymbol{h}}}

\def\w{{\boldsymbol{w}}}
\def\W{{\boldsymbol{W}}}
\def\Hb{{\boldsymbol{H}}}

\newcommand{\nn}{\nonumber}

\newcommand{\nmaj}{n_{\mathrm{maj}}}
\newcommand{\nmin}{n_{\mathrm{min}}}
\newcommand{\spu}{w^{\mathrm{use-spu}}}
\newcommand{\core}{w^{\mathrm{use-core}}}
\newcommand{\pmaj}{p_{\mathrm{maj}}}

\usepackage[nonatbib,preprint]{neurips_2022}

\usepackage[utf8]{inputenc} 
\usepackage[T1]{fontenc}    
\usepackage{hyperref}       
\usepackage{url}            
\usepackage{booktabs}       
\usepackage{amsfonts}       
\usepackage{nicefrac}       
\usepackage{microtype}      
\usepackage{xcolor}         
\usepackage{multicol}

\title{Importance Tempering:  Group Robustness for Overparameterized Models}

\author{%
  Yiping Lu\\
  ICME\\
  Stanford University\\
  Stanford, CA 94305 \\
  \texttt{yplu@stanford.edu} \\
  \And
  Wenlong Ji\\
  Department of Satistics\\
  Stanford University\\
  Stanford, CA 94305 \\
  \texttt{jwl2000@stanford.edu}\\
  \AND 
  Zachary Izzo\\
  Department of Mathematics\\
  Stanford University\\
  Stanford, CA 94305 \\
  \texttt{zizzo@stanford.edu}\\
  \And
  Lexing Ying\\
  Department of Mathematics\\
  Stanford University\\
  Stanford, CA 94305 \\
    \texttt{lexing@stanford.edu}\\
}

\begin{document}

\maketitle

\begin{abstract}
  Although overparameterized models have shown their success on many machine learning tasks, the accuracy could drop on the testing distribution that is different from the training one. This accuracy drop still limits applying machine learning in the wild. At the same time, importance weighting, a traditional technique to handle distribution shifts, has been demonstrated to have less or even no effect on overparameterized models both empirically and theoretically.  In this paper, we propose importance tempering to improve the decision boundary and achieve consistently better results for overparameterized models. Theoretically, we justify that the selection of group temperature can be different under label shift and spurious correlation setting. At the same time,  we also prove that properly selected temperatures can extricate the minority collapse for imbalanced classification. Empirically, we achieve state-of-the-art results on worst group classification tasks using importance tempering.
\end{abstract}

\section{Introduction}
\vspace{-0.1in}
Overparameterized neural networks have achieved state-of-the-art performance on numerous machine learning tasks. However, they can fail when the test data distribution differs from the training data distribution. In this paper, we consider the generalization properties of overparameterized neural networks on a typical subgroup of the data \cite{sagawa2019distributionally,sagawa2020investigation}, particularly when a certain subgroup of the data is hard to sample \cite{cao2019learning} and overparameterized neural networks become vulnerable to fitting spurious features \cite{torralba2011unbiased,buolamwini2018gender,zech2018variable,geirhos2020shortcut,xiao2020noise}.

\begin{figure}[h]
	\centering
	\vspace{-0.15in}
	\subfigure[Linear Model for Separable Data]{
		\begin{minipage}[b]{0.4\textwidth}
		\centering
			\includegraphics[width=1\textwidth]{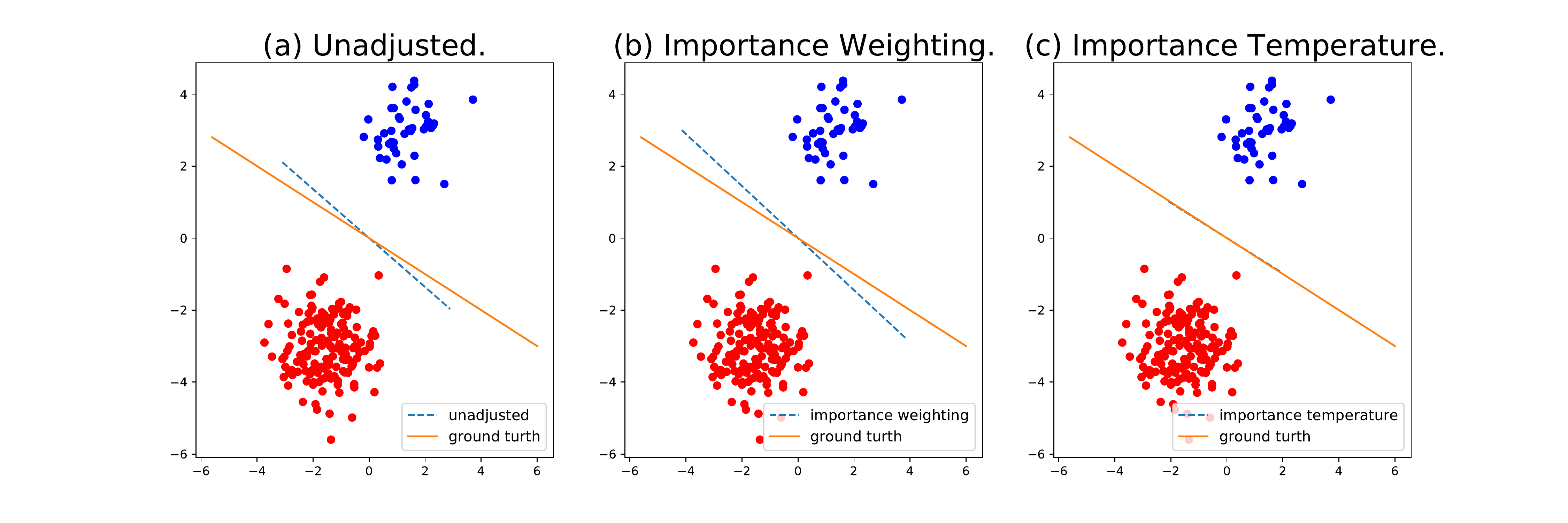}
		\end{minipage}
		\label{fig:hor_2figs_1cap_2subcap_1}
	}
    	\subfigure[Multilayer Perceptron with two hidden layers of size 200]{
    		\begin{minipage}[b]{0.4\textwidth}
    		\centering
   		 	\includegraphics[width=1\textwidth]{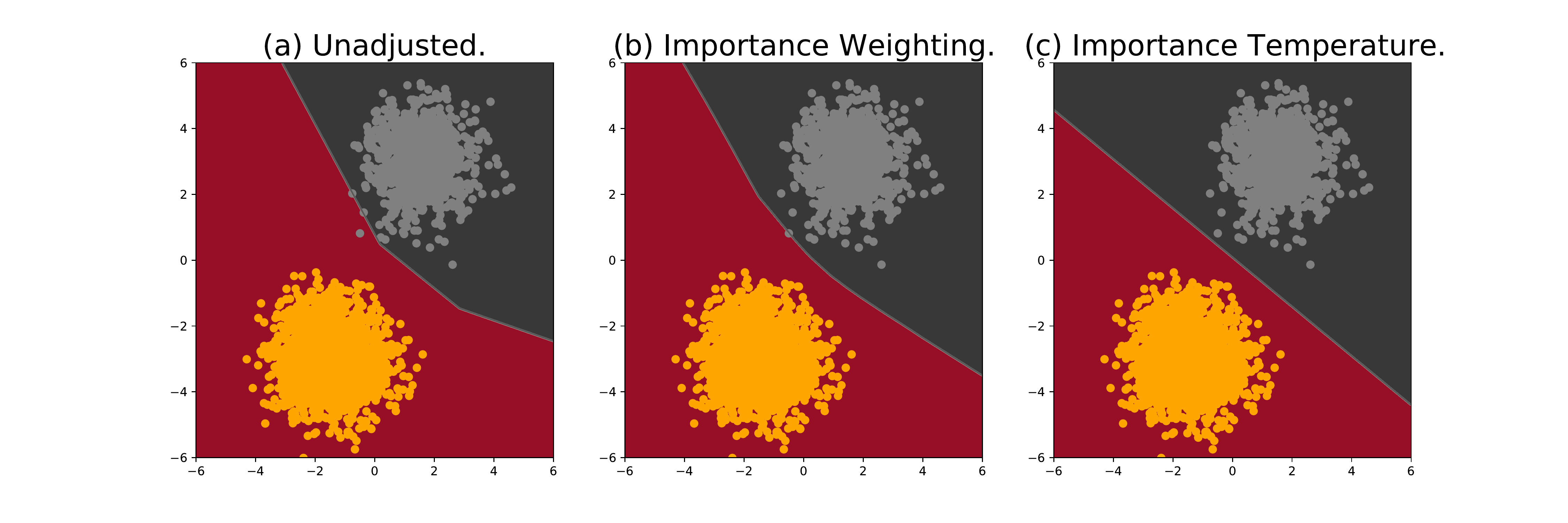}
    		\end{minipage}
		\label{fig:hor_2figs_1cap_2subcap_2}
    	}
    	\vspace{-0.1in}
	\caption{As shown in \cite{byrd2019effect}, importance weighting does not change the decision boundary, but importance tempering can. Training data points are colored according to their true labels. The learned boundary is plotted as dotted blue line in (a) and shown by the different background colors in (b).}
\vspace{-0.1in}
	\label{fig:page1figure}
\end{figure}

Importance weighting \cite{shimodaira2000improving,sugiyama2008direct,cortes2010learning} is a classical statistical technique to train machine learning models that can adapt to class imbalances by re-weighting the loss function during training. Using importance weights, one can construct an unbiased estimator of the test loss via upweighting the training data that are more likely to appear in the test data. However, recent studies show that importance weighting has little to no impact on generalization when training deep neural networks to convergence \cite{byrd2019effect,xu2021understanding,anonymous2022is}, but rather only improves optimization properties \cite{anonymous2022stochastic}. \cite{zhai2022understanding} proved that overparameterized models trained with dynamic importance weightng \cite{fang2020rethinking} also does not improve over ERM. Moreover, in the current deep learning paradigm, practitioners frequently train overparameterized models that can interpolate the training data \cite{belkin2019reconciling,belkin2021fit}. Empirically, importance weighting has an impact only if strong regularization, \emph{i.e.}, early stopping or explicit $l_2$ regularization, is applied \cite{byrd2019effect,sagawa2019distributionally}. Theoretically, it has been shown that overparameterized linear and non-linear models trained with the importance weighted exponential or cross-entropy loss converge to the max-margin model \cite{soudry2018implicit,nacson2019convergence,lyu2019gradient,chizat2020implicit,ji2020directional} and such models will ignore the importance weights \cite{xu2021understanding}.

In this paper, we address these problems by proposing an alternative to importance weighting for overparameterized models, dubbed \textit{importance tempering} (IT). Inspired by \cite{cao2019learning}, we assign different margins to the training examples from different groups by adding temperature parameters to the exponential-tailed point loss. Unlike importance weighting, which has little to no impact when the network interpolates the training data, importance tempering increases the margin for the minority class and finds a better decision boundary, as shown in Figure \ref{fig:page1figure} for a simple Gaussian mixture dataset. Our numerical experiments show that importance tempering increases worst group accuracy even when the model is overparameterized. This observation refutes the hypothesis of \cite{sagawa2020investigation}, which states that overparameterization causes deep neural networks to overfit to spurious features in the data.

\vspace{-0.05in}
\subsection{Related Works}
\vspace{-0.05in}
\textbf{Implicit Bias of Gradient Descent} To understand how gradient descent and its variants help deep learning to find solutions with good generalization performance on the test data, a recent line of research has studied the implicit bias of gradient descent in different settings. For example, gradient descent is biased towards solutions with minimum norm under $\ell_2$ loss \cite{li2018algorithmic,vaswani2020each} and will converge to large margin solutions when using exponential-tailed loss \cite{soudry2018implicit,nacson2019convergence,lyu2019gradient,chizat2020implicit,ji2020gradient}. 

\textbf{Imbalanced Classification} \cite{cao2019learning} considered a label-distribution-aware margin loss for imbalanced data classification and selected the margin to minimize the generalization bound.
\cite{anonymous2022is} considered using a polynomial tailed loss (such as the focal loss \cite{lin2017focal}) instead of an exponential-tailed or cross-entropy loss. For the loss functions they consider, importance weighting can still have an effect even for overparameterized models. In this paper, we focus on using the cross-entropy loss since this is the most commonly used loss for classification problems in practice. \cite{kini2021label,ye2020identifying,narasimhan2021training,menon2020long,wang2021seesaw,anonymous2022learning,fang2020rethinking} proposed different loss functions for imbalanced classification tasks. For a detailed discussion, we refer the reader to Remark \ref{remark:related}. Furthermore, all of these papers focused on the average classification error over all groups, while in the present paper we also address the worst group classification error. This leads to a different selection of the class margins from what \cite{cao2019learning} proposed.  
\vspace{-0.1in}
\subsection{Our Contributions}
\vspace{-0.12in}
In summary, our contributions are as follows:
\vspace{-0.1in}
\begin{itemize}
\setlength{\itemsep}{0pt}
\setlength{\parsep}{0pt}
\setlength{\parskip}{0pt}
    \item We introduce importance tempering to fix the ineffectiveness of importance weighting \cite{byrd2019effect,xu2021understanding} for overparameterized models. Theoretically, we prove that using importance tempering with a homogeneous neural network will result in the assignment of different margins to each group \cite{cao2019learning} via the implicit bias of (stochastic) gradient descent on exponential-tailed loss \cite{soudry2018implicit,lyu2019gradient,ji2020directional}.
    \item We discuss the impact of importance tempering on the recently discovered phenomenon of neural collapse \cite{papyan2020prevalence} on imbalanced datasets. In particular, we show that importance tempering can fix minority collapse \cite{fang2021layer} for overparameterized models.
    We also find that it is consequential whether importance tempering is applied to the last layer \emph{features} or \emph{classifier}. These two settings lead to different geometries for the last layer features, from which we conclude that importance tempering should be applied to the last layer classifier but not the features.
    \item We conduct experiments on two types of distribution shifts. We find that the optimal importance tempering varies for different types of distribution shifts, which is in contrast to the common practice of selecting an importance weight equal to the imbalance ratio. We also show that importance tempering consistently improves the worst group accuracy even when the model is larger, refuting the hypothesis of \cite{sagawa2020investigation} that overparameterization causes deep neural networks to overfit to spurious features in the data.
\end{itemize}
\vspace{-0.05in}
\section{Importance Tempering}
\vspace{-0.05in}
In this section, we introduce our method, importance tempering (\textbf{IT}), which can be viewed as an analogue of importance weighting for overparametrized models trained with an exponential-tailed loss. We apply different temperatures to the exponential loss for different data points to control the model's level of confidence for each data point. Specifically, we show that IT will assign different classification margins to different subgroups of the data. Our proofs use techniques from \cite{soudry2018implicit,lyu2019gradient,ji2020directional}.

\vspace{-0.05in}
\subsection{Problem setup}
\vspace{-0.05in}
We assume that data points $x=\{(x_i,y_i,g_i)\}_{i=1}^{n}$ are sampled from $n_g$ groups.  Here $x_i\in\mathbb{R}^d$ are the features, $y_i\in\mathbb{R}$ is the label, and $g_i\in \{1,2,\ldots,n_g\}$ is the corresponding group label. Empirical risk minimization (ERM) aims to optimize $\mathcal{L}^{\text{ERM}}(\theta)=\frac{1}{n}\sum_{i=1}^n\exp(-y_iq(x_i,\theta))$, where $q(x,\theta)$ denotes the output of a neural network on input $x$ with parameters $\theta$. For simplicity, in this section, we consider a binary classification setting, \emph{i.e.}, $y\in\{-1,1\}$ and our prediction is given by the sign of $q(x,\theta)$. We will discuss how to use importance tempering with cross-entropy loss for multi-class classification problems in Section~\ref{section:multiclass}. IT modifies the ERM setting by adding temperature parameters for each group in the data:
$$
\mathcal{L}^{\text{IT}}(\theta)=\frac{1}{n}\sum_{i=1}^n\exp(-y_iq(x_i,\theta)f[g_i]).
$$
where $f[g_i]$ is the importance weight of group $g_i$. We then train $\theta$ by minimizing $\mathcal{L}^{\text{IT}}$.

\begin{remark} \label{remark:related}
Adding a temperature parameter was first introduced for facial recognition in \cite{guo2017one,khan2019striking}.  Independent work \cite{ye2020identifying,kini2021label} also introduced a temperature for the label shift problem. Our paper is different from these papers from two perspectives. First, these papers only address classification error without distribution shift. In this paper, we mainly discuss the impact of importance tempering on an overparameterized model's worst group performance. At the same time, the theory in \cite{kini2021label} only considers the two-class classification problem with label shift. In Section \ref{section:multiclass}, we show that the geometry of multi-class problems can be very different. Second, label shift is a special case of the problem we consider. In particular, the group variables $g_i$ can be different from the classes, which leads to a different selection of the temperature. 
\end{remark}
\vspace{-0.05in}
\subsection{importance tempering corrects the implicit bias}
\vspace{-0.1in}
In this section, following \cite{lyu2019gradient}, we will show that training an overparametrized homogeneous neural network with IT results in the solution of a cost-sensitive SVM problem \cite{shawe1999optimizing,fumera2002cost}. We make the following assumption on our model:
\begin{assumption}[Homogeneous model]  
\label{asm: homogeneity}
There exists a constant $L>0$ such that
$$
q(x,\alpha\theta)=\alpha^Lq(x,\theta), \forall \alpha>0.
$$
\end{assumption}
\vspace{-0.1in}
This assumption includes $L$-layer fully-connected and convolutional neural networks with ReLU or LeakyReLU activations as widely used examples. For such a model, we can establish the following result:
\vspace{-0.06in}
\begin{theorem}[Informal]
\label{thm: implicit bias}
For a homogeneous model $q(x,\theta)$ with some regularity conditions, let $\theta(t)$ denote the model parameters trained with gradient flow at time $t$. If there exists a time $t_0$ such that $\mathcal{L}^{\text{IT}}(\theta(t_0))<\frac{1}{n}$, then any limit point of $\frac{\theta(t_0)}{\|\theta(t_0)\|}$ is along the direction of (i.e., a scalar multiple of) a Karush-Kuhn-Tucker (KKT) point of the following minimum-norm separation problem:
$$
\min_\theta \|\theta\| \: \text{ s.t. } \:
y_iq(x_i,\theta)\ge 1/f[g_i], \hspace{.1in} i = 1,\ldots,n.
$$
\end{theorem}
\vspace{-0.1in}
\section{Importance Tempering for Label Shift}
\vspace{-0.1in}
In this section, we focus on the problem of label shift. 
In this case, the subgroups coincide precisely with the different classes (labels) of the data. This setting has been well studied for underparameterized models \cite{storkey2009training,han2018co,garg2020unified}. We provide the corresponding theory for overparameterized models in the label shift setting. Our method is compared with the reweighting-based method \cite{youbi2021simple} in Table \ref{table:result}. In this setting, the ratio of sample sizes across different classes is different for training and testing. For example, training data from a certain group of people may be extremely rare due to a bias in the data collection procedure, but we still want our model to perform well for this under-sampled group after we train and deploy it. Formally, consider a $K$-class classification problem, where $n_i,i\in[K]$ samples in the training data are drawn from class $i$ (sampled from distribution $p_i$). In the imbalanced setting, we may expect that the $n_i$ are of vastly different sizes.
\vspace{-0.15in}
\subsection{A Generalization Theorem for the Binary Case}
\vspace{-0.10in}
Below we provide a generalization bound in the binary label shift setting, which suggests setting the temperature as the square root of imbalance ratio. The proof of the theorem is shown in Appendix \ref{appendix:generalization}.
\vspace{-0.05in}
\begin{theorem}(Informal)\label{thm:sqrt} Let $\mathcal{F}$ denote the function class of two-layer 2-homogeneous neural networks, and let $\mathcal{C}(\mathcal{F})$ denote some proper complexity measure of the model class. If we fix the sum of temperatures $\sum_i f[i]$ to be a constant, then with high probability over the randomness of the training data, we have
$$
\max_i \mathbb{P}_{x\sim p_i}\left[y_iq(x,\theta)\le0\right]\lesssim \max_i {f[i]}\sqrt{\frac{\mathcal{C}(\mathcal{F})}{n_j}}.
$$
Furthermore, selecting $f[i]\propto {\sqrt{n_i}}$ minimizes the resulting bound. 
\end{theorem}
\begin{remark}
Similarly, if we consider a balanced test distribution, the average test error can be bounded by $\sum_i {f[i]}\sqrt{\frac{\mathcal{C}(\mathcal{F})}{n_j}}$. In this situation, $f[i]\propto {\sqrt{n_i}}$ still minimizes the bound. The corresponding empirical results are also plotted in Figure~\ref{figure:temperature}.
\end{remark}
\vspace{-0.1in}
We tested our theory on both the CIFAR-10 \cite{krizhevsky2009learning} and Fashion MNIST \cite{xiao2017fashion} datasets with a ResNet-32 \cite{he2016deep} model. We specify minority and majority groups with $n_1$ and $n_2$ training points, respectively, and set the importance tempering for the minority group to be $\left(\frac{n_1}{n_2}\right)^\gamma$, where $\gamma$ is a hyperparameter to be tuned. The majority group has an importance tempering equal to 1. (Note that this is equivalent to setting the temperature for each group.) We then vary $\gamma$ from 0 to 1. The experiment confirms our theory as the model achieves the best performance with $\gamma \approx 0.5$. For more details, we refer to Section~\ref{section:exp1}.
\vspace{-0.15in}
\subsection{Multi-class and Neural Collapse}
\label{section:multiclass}
\vspace{-0.10in}
Recently, \cite{papyan2020prevalence} observed that during the terminal phase of training (\emph{i.e.}, the stage after achieving zero training error) over a balanced dataset, the features for data points within the same class collapse to their mean, and the feature means for each class will converge to the simplex equiangular tight frame (ETF). This \emph{neural collapse} \cite{papyan2020prevalence} phenomenon enables us to understand the benefit of training after achieving zero training error to
achieve better performance in terms of generalization and robustness. For imbalanced datasets, \cite{fang2021layer} discovered that the minority classes are not distinguishable in terms of their last layer classifiers when the imbalance ratio exceeds a threshold. This phenomenon is known as \emph{minority collapse}, and it fundamentally limits the performance of feature-learning models for the minority classes. In this section, we aim to show that IT can extricate the minority collapse for multi-class classification problems.

For the multi-class classification problem, there are two ways to introduce importance tempering. Following \cite{fang2021layer}, in this section, we study the two resulting loss functions and explore their differences in the extremely imbalanced limit. We analyze this setting by way of the layer peeled model \cite{fang2020rethinking,zhu2021geometric,ji2021unconstrained} as follows. A standard neural network architecture computes an output of the form
\begin{equation}
f\left(\boldsymbol{x} ; \W_{full} \right)=\W_{L} \sigma\left(\boldsymbol{b}_{L-1}+\W_{L-1} \sigma\left(\cdots \sigma\left(\boldsymbol{b}_{1}+\W_{1} \boldsymbol{x}\right)\right)\right)
\end{equation} 
In the layer peeled model, for each data point in the dataset $\bigcup_{k=1}^K\{\x_{k,i},i=1,\cdots,n_k\}$, its last layer representation $\h_{k,i}=\sigma\left(\boldsymbol{b}_{L-1}+\W_{L-1} \sigma\left(\cdots \sigma\left(\boldsymbol{b}_{1}+\W_{1} \boldsymbol{x}_{k,i}\right)\right)\right)\in\mathbb{R}^d$ is considered as a free variable which we can choose directly. The same holds for the last layer classifier $\W\in\mathbb{R}^{K \times d}=\W_L=[\w_1,\w_2,\cdots,\w_K]^\top$ which will be applied to the representations $\h_{k,i}$. The unconstrained layer-peeled model (ULPM) \cite{ji2021unconstrained} simplifies the cross entropy loss as
\begin{equation}
    \label{loss:original}
    \mathcal{L}(\W,\Hb) =  -\sum_{k=1}^{K}\sum_{i=1}^{n_k} \log\left( \frac{\exp(\w_k^\top \h_{k,i})}{\sum_{j=1}^K\exp(\w_j^\top \h_{k,i})} \right).
\end{equation}

\begin{figure}[h]
 \vspace{-0.2in}
	\centering
	\subfigure[CIFAR 10, Average]{
			\includegraphics[width=0.2\textwidth]{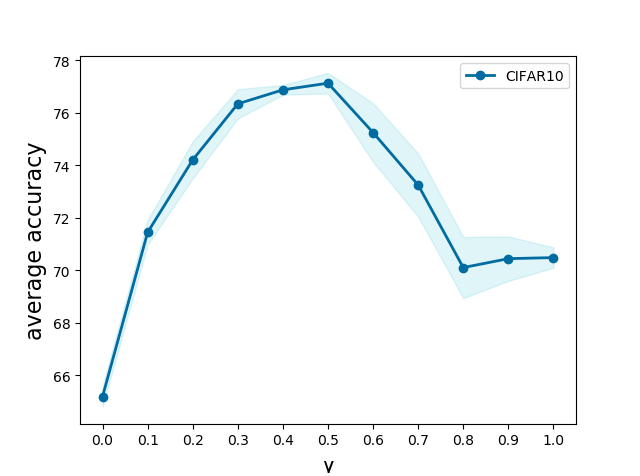}
		\label{fig:cifaravg}
	}
	\subfigure[CIFAR 10, Worst]{
			\includegraphics[width=0.2\textwidth]{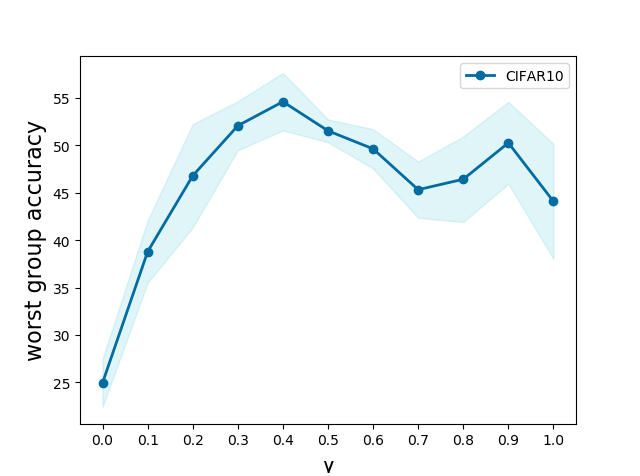}
		\label{fig:cifarworst}
	}
	\subfigure[Fashion MNIST, Average]{
			\includegraphics[width=0.2\textwidth]{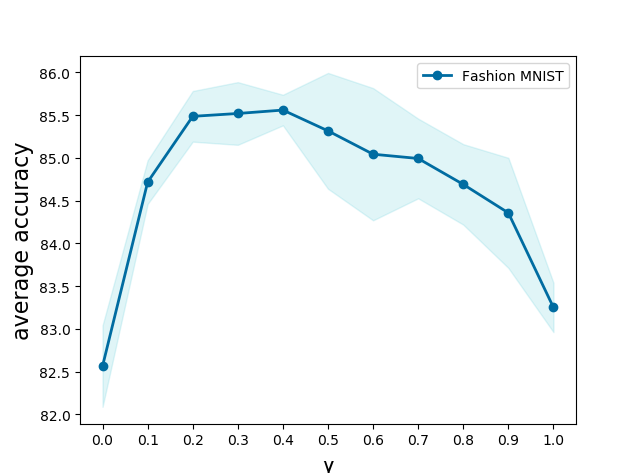}
		\label{fig:mnistavg}
	}
	\subfigure[Fashion MNIST, Worst]{
			\includegraphics[width=0.2\textwidth]{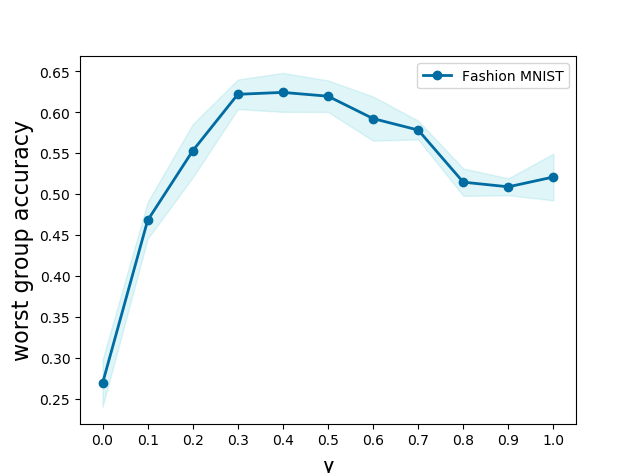}
		\label{fig:mnistworst}
	}\vspace{-0.15in}
	\label{figure:temperature}
	\caption{Effect of the minority class temperature $\left(\frac{n_1}{n_2}\right)^\gamma$ on CIFAR-10 and Fashion MNIST with an imbalanced of 1:100. The best performance on both average accuracy and worst group accuracy occur when $\gamma \approx 0.5$.}
	\vspace{-0.12in}
\end{figure}
Following \cite{buda2018systematic,cao2019learning,fang2021layer}, we consider the step imbalance setting in this section. We consider two different class sizes during training time: the majority classes each contain $n_{A}$ training examples ($n_1=n_2=\cdots=n_{[K/2]}=n_A$), and the minority classes each contain $n_B$ training examples ($n_{[K/2]+1}=n_{[K/2]+2}=\cdots=n_{K}=n_B$). We call $R:=\frac{n_A}{n_B}$ the imbalance ratio. At test time, however, the classes are balanced, i.e., each class has the same number of data points.

To incorporate importance tempering into the cross-entropy loss, we can either add the temperature to the features $\h$ or to the last layer classifier $\w$. Introducing the temperature at these different positions results in two different objective functions:
\vspace{-0.15in}
\begin{multicols}{2}
\begin{equation}
\footnotesize
		\label{loss:h}
		\mathcal{L}^{\text{IT(H)}}(\theta)=-\sum_{k=1}^{K}\sum_{i=1}^{n_k}\log\frac{\exp(\w_k^\top \lambda_{k} \h_{k,i})}{\sum_{j=1}^{K}\exp(\w_j^\top \lambda_{k}\h_{k,i})},
\end{equation}
\begin{equation}
\footnotesize
	\label{loss:w}
	\mathcal{L}^{\text{IT(W)}}(\theta)=-\sum_{k=1}^{K}\sum_{i=1}^{n_k}\log\frac{\exp(\lambda_{k}\w_k^\top \h_{k,i})}{\sum_{j=1}^{K}\exp(\lambda_{j}\w_j^\top \h_{k,i})}.
\end{equation}
\end{multicols}

\begin{remark}
The ambiguity in where to add IT only appears in the case of label shift. For a general worst group problem (i.e., where the groups are not necessarily aligned with the labels), one can only add the temperature to the last layer features $\h$. In this case, the objective function is different from the independent work of \cite{kini2021label}.
\end{remark}
	
\begin{table*}[]
\centering
\footnotesize
\caption{Effect of incorporating importance tempering on the last layer features vs. classifiers for imbalanced CIFAR-10. We find that introducing the temperature on the last layer classifier is better.}
\begin{tabular}{c||cc|cc|cc}
\hline
\hline
\multirow{2}{*}{\thead{\textbf{Imbalance}\\
\textbf{ Ratio}}} & \multicolumn{2}{c|}{\textbf{Vanilla}}                        & \multicolumn{2}{c|}{\textbf{Temperature over feature}}       & \multicolumn{2}{c}{\textbf{Temperature over classifier}}    \\ \cline{2-7} 
                                         & \multicolumn{1}{c|}{\textbf{Worst Group}} & \textbf{Average} & \multicolumn{1}{c|}{\textbf{Worst Group}} & \textbf{Average} & \multicolumn{1}{c|}{\textbf{Worst Group}} & \textbf{Average} \\ \hline\hline
\textbf{1:10}                                     & \multicolumn{1}{c|}{64.3}                     &          85.65        & \multicolumn{1}{c|}{67.2}                     &         86.27         & \multicolumn{1}{c|}{72.7}                     &      87.89            \\ \hline
\textbf{1:100}                                    & \multicolumn{1}{c|}{21.5}                 & 67.52            & \multicolumn{1}{c|}{25.9}                 & 70.11            & \multicolumn{1}{c|}{57.2}                 & 76.65            \\ \hline
\end{tabular}
\label{table:tempposition}
\vspace{-0.2in}
\end{table*}
\vspace{-0.1in}
\subsubsection{Theoretical Results}
\vspace{-0.1in}
Here we show how the choice of importance tempering and the position at which it is introduced can impact the geometry of the last layer features and classifiers in the extremely imbalanced setting (\emph{i.e.}, $R\rightarrow\infty$) considered by \cite{fang2021layer}. We first link the converged solution of gradient flow on homogeneous neural networks to the KKT point of the corresponding minimum-norm separation problem. We then consider the global solution of the cost-sensitive SVM problem to study the geometry of the last layer features.  From \cite{lyu2019gradient,ji2021unconstrained}, we know that the gradient descent dynamics of objective function \eqref{loss:original} converges to a KKT point of
\begin{equation}
\begin{aligned}
\label{SVM:or}
&\min _{\W, \Hb} \frac{1}{2}||\W||_F^2+\frac{1}{2}||\Hb||_F^2\quad
s.t.& \w_k^\top \h_{k,i}-\w_j^\top \h_{k,i} \geq 1,\quad  k\not=j\in[K],i\in [n_k],
\end{aligned}
\end{equation} 
the gradient descent dynamics of \eqref{loss:h} converges to a KKT point of
\begin{equation}
\begin{aligned}
\label{SVM:h}
&\min _{\W, \Hb} \frac{1}{2}||\W||_F^2+\frac{1}{2}||\Hb||_F^2\quad
s.t.& \lambda_k\w_k^\top \h_{k,i}-\lambda_k\w_j^\top \h_{k,i} \geq 1,\quad  k\not=j\in[K],i\in [n_k],
\end{aligned}
\end{equation}
and the gradient descent dynamics of \eqref{loss:w} converges to a KKT point of
\begin{equation}
\begin{aligned}
\label{SVM:w}
&\min _{\W, \Hb} \frac{1}{2}||\W||_F^2+\frac{1}{2}||\Hb||_F^2\quad
s.t.& \lambda_k\w_k^\top \h_{k,i}-\lambda_j\w_j^\top \h_{k,i} \geq 1,\quad  k\not=j\in[K],i\in [n_k].
\end{aligned}
\end{equation}

\begin{theorem}\label{theorem:collapse}(Informal) Assume $R:=\frac{n_A}{n_B}\rightarrow\infty$ and select the temperature as Theorem \ref{thm:sqrt} suggests, \emph{i.e.} $\lambda_j=\sqrt{n_j}$. Then the following statements hold:
\begin{itemize}
    \item \textbf{(a)} If the global solution $(\Hb^\ast,\W^\ast)$ of (\ref{SVM:or}) has a limit, then the limit is a minority collapse solution, \emph{i.e.} $\lim_{R\rightarrow \infty} w_k^\ast-w_{k'}^\ast = 0, \text{ for all } K/2<k<k'\le K.$
    \item \textbf{(b)} The global solution $(\Hb^\ast,\W^\ast)$ of (\ref{SVM:h}) converges to the neural collapse solution, \emph{i.e} a simplex ETF solution: the vectors of the class means (after centering by their global mean) converge to vectors of equal length, form equal-sized angles between any given pair, and are the maximally
pairwise-distanced configuration subject to having the previous two properties $\cos(\bar{{\h}}_k,\bar{{\h}}_j)=-\frac{1}{K-1},\quad||\bar{{\h}}_k||=||\bar{{\h}}_j||,\quad k\not = j.$
    \item \textbf{(c)} If the global solution of problem (\ref{SVM:w}) has a directional limit, the directional limit of the global solution $(\Hb^\ast,\W^\ast)$ of (\ref{SVM:w}) satisfies $\lim_{R\rightarrow \infty}\cos(\bar{{\h}}_k,\bar{{\h}}_j)=-\frac{1}{\frac{K}{2}-1},\quad||\bar{{\h}}_k||=||\bar{{\h}}_j||,$ for all $K/2+1\le k\not = j\le K$.
\end{itemize}
\end{theorem}
\vspace{-0.2in}
\paragraph{Discussion} Theorem \ref{theorem:collapse} shows that tempering the last layer features $h$ enables the class means (centered at the global-mean) to form the largest possible equal-sized angles between any pair of class means, while tempering the last layer classifier $w$ only enlarges the angles between the minority classes. This leads to larger angles (from $\arccos\left(-\frac{1}{K-1}\right)$ when tempering $h$ to $\arccos\left(-\frac{1}{\frac{K}{2}-1}\right)$ when tempering $w$) between the minority class vectors and thus better results on the minority classes. At the same time, $\arccos\left(-\frac{1}{\frac{K}{2}-1}\right)$ is the largest possible angle that can be achieved when all of the minority feature vectors form an equiangular frame. 

\begin{wrapfigure}{r}{0.5\textwidth}
	\centering
	\vspace{-0.8in}
	\subfigure[Average angle between majority groups]{
			\includegraphics[width=0.2\textwidth]{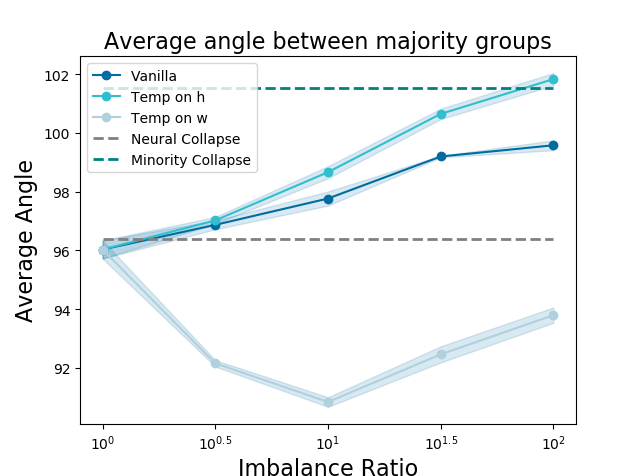}
		\label{fig:majorcollapse}
	}
	\subfigure[Average angle between minority groups]{
			\includegraphics[width=0.2\textwidth]{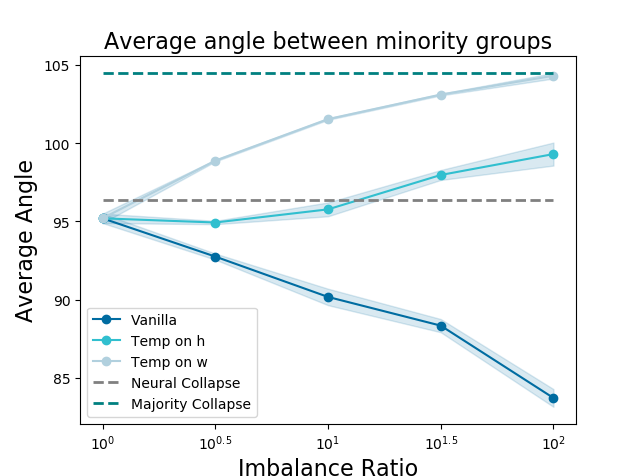}
		\label{fig:minorcollapse}
	}\vspace{-0.15in}
	\caption{Average angle of majority group and minority group under different imbalance ratio on Fashion MNIST. As our theory suggests, the angle between minority groups is always roughly $\arccos\left(\frac{1}{K-1}\right)$ when we add IT to the features, while the angle between the minority groups converges to $\arccos\left(\frac{1}{\frac{K}{2}-1}\right)$ if we add IT to the last layer classifier.}
	\label{figure:collapse}
	\vspace{-0.2in}
\end{wrapfigure}

\subsection{Experimental Results}
\label{section:exp1}
\vspace{-0.1in}
Following \cite{fang2021layer}, we test our algorithms on the FashionMNIST \cite{xiao2017fashion} and CIFAR-10 \cite{krizhevsky2009learning} datasets. We choose the first five classes as majority classes and make the second five classes into minority classes by subsampling. We test adding IT to either the features or the last layer classifier, with the temperature proportional to the square root of the number of samples. The results are shown in Table~\ref{table:tempposition}. \cite{fang2021layer} has shown that importance weighting can also mitigate minority collapse. \cite{fang2021layer} only trains the network for 300 epochs. As shown in \cite{byrd2019effect}, after training the network for 1000 epochs, the effect of importance weighting will become negligible. In this paper, we mainly consider fully-trained networks and leave the regularization of early stopping for future work.

As our theory suggests, one should add the temperature to the last layer linear classifier; this is in agreement with the results of \cite{khan2019striking}. However, we find that the class feature means do not converge to an equiangular tight frame as \cite{papyan2020prevalence} suggested. The effect of the imbalance ratio on angles between majority/minority classes is shown in Figure \ref{figure:collapse}, where the constants $\arccos\left(\frac{1}{K-1}\right)$ (gray line) and $\arccos\left(\frac{1}{\frac{K}{2}-1}\right)$ (green line) are marked for comparison. Our experimental results matches what our theory (Theorem \ref{theorem:collapse}) predicts: adding IT to the last layer classifier leads to the largest possible angle in the extremely imbalanced limit.


\vspace{-0.1in}
\section{Importance Tempering for Spurious Correlations} 
\vspace{-0.1in}
Worst group accuracy \cite{sagawa2019distributionally,sagawa2020investigation,liu2021just} is a relevant metric for reducing the the reliance of machine learning models on spurious correlations \cite{torralba2011unbiased,geirhos2020shortcut,arjovsky2019invariant}.  In this setting, each example is composed of the input $x$, a label (core attribute) $y \in \mathcal{Y}$, and a spurious attribute $a\in\mathcal{A}$. Each data point belongs to a group $g=(y,a)\in\mathcal{Y}\times\mathcal{A}$.  Spurious correlations refer to correlations between the label and the spurious attribute for a particular group (which in general will not generalize across different groups). Here we focus on the binary case $\mathcal{Y}=\{0,1\}$ and $\mathcal{A}=\{0,1\}$. Following \cite{sagawa2019distributionally,sagawa2020investigation,liu2021just}, we test our objective function with a ResNet-50 \cite{he2016deep} on the CelebA and Waterbird dataset and Bert \cite{devlin2018bert} on the MultiNLI dataset. In CelebA, the label $y$ is whether or not the image contains a person with blonde hair. The spurious attribute is the gender of the person in the image. In the Waterbird dataset, we aim to classify land and water birds. Here, the spurious attribute is the background of the image (land or water).
For natural language processing, \cite{gururangan2018annotation} recently found that there is a spurious correlation between contradictions and the presence of negation words such as nobody, no, never, and nothing. We use the MultiNLI dataset to distinguish between entailed, neutral, and contradictory examples and aim to achieve good accuracy regardless of the spurious attribute (presence or absence of negation words). More details on these datasets can be found in \cite{sagawa2019distributionally}. The experiment details are shown in Appendix \ref{appendix:expdetail}. As shown in Table \ref{table:result}, importance tempering achieves comparable results with Group DRO \cite{sagawa2019distributionally}.

\begin{table*}[h]
\footnotesize
\caption{Comparison of Empirical Risk Minimization (ERM), Importance Weighting (IW), group DRO, and importance tempering (IT) models on several group shift and spurious correlation benchmarks. Large Models refers to results using WideResNet-50 for computer vision and Bert Large for natural language processing.}
\begin{tabular}{|ll||l|l|l|l|l|ll|l|}
\hline\hline
\multicolumn{2}{|c||}{\multirow{2}{*}{\textbf{\normalsize Dataset}}}
 & \multicolumn{8}{c|}{\textbf{Worst-Group Accuracy}}                                                       \\ \cline{3-10} 
 && {ERM}& {ERM} & {IW} &{IW}&{\thead{\scriptsize Group \\ \scriptsize DRO}}& \multicolumn{1}{l||}{\thead{\scriptsize Group \\ \scriptsize DRO}} & IT&IT\\ \hline\hline

\multirow{2}{*}{\textbf{Label Shift} 1:10}&Fashion MNIST & {69.9} & {\cellcolor{mygray}0}  & 73.2& {\cellcolor{mygray}0}               &      - & \multicolumn{1}{l||}{\cellcolor{mygray}-}&       \textbf{79.0}     &\cellcolor{mygray}0         \\ 
&CIFAR10& {64.3}   & {\cellcolor{mygray}0}  & 71.3 & {\cellcolor{mygray}0}           &         -& \multicolumn{1}{l||}{\cellcolor{mygray}-}          &   \textbf{72.7}        &      \cellcolor{mygray}0  \\ \hline

\multirow{2}{*}{\textbf{Label Shift} 1:100}&Fashion MNIST & {27.7} & {\cellcolor{mygray}0} & 59.8 & {\cellcolor{mygray}0}               &      - & \multicolumn{1}{l||}{\cellcolor{mygray}-}&      \textbf{64.7}     &\cellcolor{mygray}0         \\ 
&CIFAR10& {21.5}   & {\cellcolor{mygray}0}  & 33.2 & {\cellcolor{mygray}0}           &         -& \multicolumn{1}{l||}{\cellcolor{mygray}-}          &   \textbf{57.2 }        &      \cellcolor{mygray}0  \\ \hline

\multirow{3}{*}{ \thead{\textbf{Spurious}\\ \textbf{Correlations}} }&CelebA                            & {41.1} & {\cellcolor{mygray}47.8} & 82.1 & {\cellcolor{mygray}83.8}               & 88.3      & \multicolumn{1}{l||}{\cellcolor{mygray}88.9}          &     89.1        &\cellcolor{mygray}\textbf{90.1}           \\ 
&Waterbird & {60.0}   & {\cellcolor{mygray}63.7}  & - & {\cellcolor{mygray}88.0}           & 86.0          & \multicolumn{1}{l||}{\cellcolor{mygray}\textbf{91.4}}          &   88.7          &      \cellcolor{mygray}89.5     \\ 
&MultiNLI & {65.7}   & {\cellcolor{mygray}-}  & 64.8 & {\cellcolor{mygray} -}           & \textbf{77.7}          & \multicolumn{1}{l||}{\cellcolor{mygray}-}          &   75.9         &      \cellcolor{mygray}-     \\ 
\hline
\multirow{2}{*}{\textbf{Large Models}}&CelebA                            & {76.7} & {\cellcolor{mygray}77.8 } & 86.8 & {\cellcolor{mygray}88.5}               & 87.4      & \multicolumn{1}{l||}{\cellcolor{mygray}87.6}          &     \textbf{90.6}    &\cellcolor{mygray}89.8         \\ 
&MultiNLI                           & {74.0} & {\cellcolor{mygray}-} & 74.3 & {\cellcolor{mygray}-}               &   76.9    &  \multicolumn{1}{l||}{\cellcolor{mygray}-}          &    \textbf{78.9}    &\cellcolor{mygray}-       \\ \hline \hline
\multicolumn{2}{|c||}{Strong $\ell_2$ Regularization}       & {}    & {\cellcolor{mygray}\checkmark} &       &{\cellcolor{mygray}\checkmark}         &    & \multicolumn{1}{l||}{\cellcolor{mygray}\checkmark}          &      &    \cellcolor{mygray}\checkmark              \\ \hline

\hline
\end{tabular}
\vspace{-0.04in}
\label{table:result}
\vspace{-0.2in}
\end{table*}

\begin{figure}
	\centering
	\subfigure[Synthetic Data in \cite{sagawa2020investigation}]{
		\begin{minipage}[b]{0.4\textwidth}
			\includegraphics[width=1\textwidth]{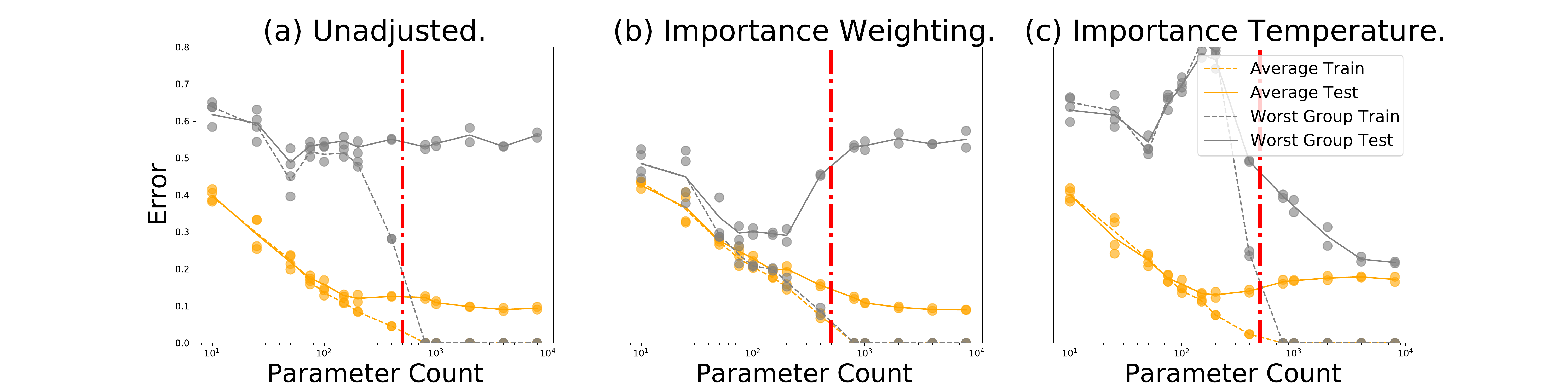}
		\end{minipage}
		\label{fig:syndd}
	}
    	\subfigure[CelebA]{
    		\begin{minipage}[b]{0.4\textwidth}
   		 	\includegraphics[width=1\textwidth]{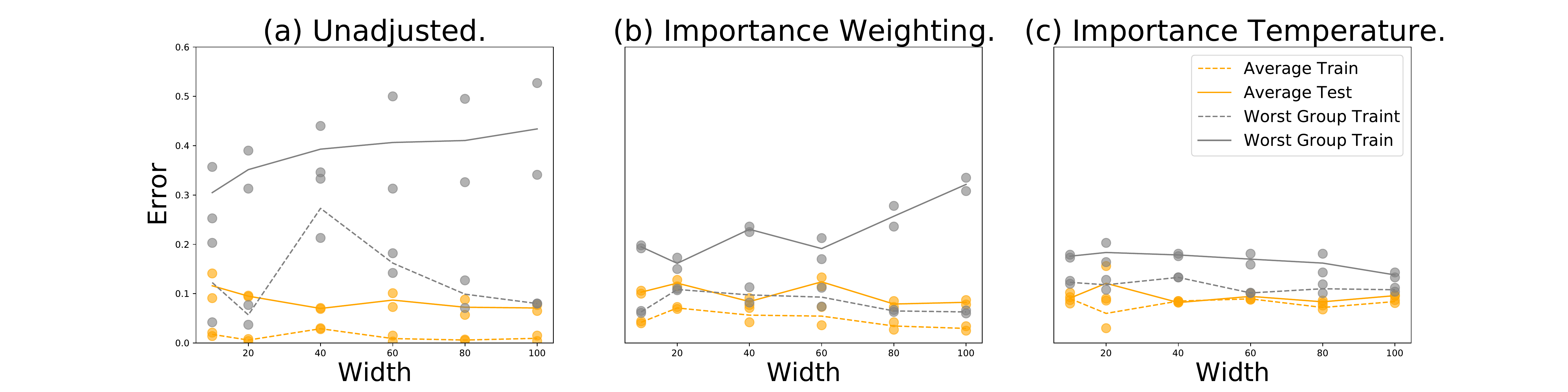}
    		\end{minipage}
		\label{fig:Celebadd}
    	}
    	\vspace{-0.15in}
	\caption{Overparameterization hurts the worst-group accuracy when the model is trained by ERM and importance weighting objectives. However, overparameterization still helps when importance tempering is applied. We plot the mean error over 2 replicates. The red line here indicates the interpolation threshold.}
	\label{fig:dd}
	\vspace{-0.2in}
\end{figure}

\subsection{Importance Tempering Cooperates with Overparameterization}
\vspace{-0.05in}
It has recently been observed \cite{opper1995statistical,belkin2019reconciling} that increasing model size beyond zero training error, \emph{i.e.} overparameterization, can lead to better test error, which is commonly referred to as the “double descent” phenomenon. However, \cite{sagawa2020investigation} showed that increasing model size well beyond the point of zero training error can hurt test error on minority groups when there are spurious correlations in the data, and hypothesized that the inductive bias towards memorizing fewer examples hurts accuracy for the minority group. Below we will show that importance tempering allows us to refute this hypothesis by changing the importance of memorization for each group.

\textbf{Synthetic Experiment Setup} We first test the impact of overparameterization on the synthetic dataset proposed in \cite{sagawa2020investigation}. In this case, both the labels and spurious attributes are $\pm1$: $\mathcal{Y} = \mathcal{A} = \{\pm1\}$. Consider two equally-sized minority groups with $a=-y$ and two equally-sized majority groups with $a=y$. In addition, every input is composed of core features $x_\text{core}\in\mathbb{R}^d$ and spurious features $x_\text{spur}\in\mathbb{R}^d$, \emph{i.e.} $x=[x_\text{core},x_\text{spur}]\in\mathbb{R}^{2d}$. We assume that both the core and spurious features are noisy and formally are sampled according to
\begin{equation}
    \begin{aligned}
    x_\text{core}|y\sim \mathcal{N}(y\mathrm{1},\sigma_\text{core}^2 I_d);x_\text{spur}|y\sim \mathcal{N}(a\mathrm{1},\sigma_\text{spu}^2 I_d),
    \end{aligned}
    \label{eq:dataset}
\end{equation}
where $\sigma_{\text{core}}^2,\sigma_{\text{spu}}^2$ are the variance of the core and spurious features. Consider logistic regression on ReLU random features $\text{ReLU}(Wx)\in \mathbb{R}^m$ \cite{mei2019generalization,montanari2019generalization}, where $W\in\mathbb{R}^{m\times 2d}$ is a random matrix with each row sampled uniformly from the unit sphere $\mathcal{S}^{2d-1}$. We set the number of training data $n=3000$ and dimension $d=100$. Setting the same hyperparameters as \cite{sagawa2020investigation}, we vary the random feature model size by increasing the number of random features from 10 to 10,000. The average and worst group test results are shown in Figure~\ref{fig:syndd}.

\textbf{CelebA} Following the experiment setting in \cite{sagawa2020investigation}, we train a ResNet-10 model \cite{he2016deep} for 50 epochs, varying model size by increasing the network width from 10 to 100 as in \cite{nakkiran2019deep}. The average and worst group test results are shown in Figure~\ref{fig:Celebadd}.
Unlike \cite{sagawa2020investigation} reporting the fully trained model, we report the result of the model early-stopped at the epoch with the best worst-group validation accuracy.  Importance weighting achieves a best worst-group test error of 85.0\% at width 20. If importance tempering is used instead, the best worst-group test error of 86.7\% is achieved at width 100. Thus {\em overparameterization still helps with generalization when importance tempering is used}. We also record the best epoch numbers and report them in Figure~\ref{fig:es}. We find that training longer in order to explore the larger parameter space only helps when importance tempering is used.

These results inspired us to use importance temperature with even larger models to further push the state of the art. We used WideResNet-50 \cite{zagoruyko2016wide} for CelebA and Bert Large \cite{devlin2018bert} for MultiNLI. In Table~\ref{table:result}, larger models consistently improved the result when IT is used. To the best of the authors' knowledge, these results give the new state-of-the-art performance on these datasets.

\vspace{-0.1in}
\subsection{How Does Importance Tempering Help?}
\vspace{-0.1in}
We return to the question of how importance tempering can help overparameterized models learn patterns that generalize to both majority and minority groups, rather than learning spurious correlations and simply memorizing the minority group. Here, we first re-investigate the intuitive story and the toy dataset in \cite{sagawa2020investigation}. Based on the story and theory, we discuss why importance tempering can avoid learning spurious correlations and how different factors will affect the selection of the temperature. 

\textbf{The intuitive story in \cite{sagawa2020investigation}.} To answer the question of what makes overparameterized models memorize the minority instead of learning generalizable patterns, \cite{sagawa2020investigation} hypothesize that the \emph{inductive bias of overparameterized models favors memorizing as few points as possible}, \emph{e.g.} by exploiting variations due to noise in the features. Consider a model that takes advantage of the fact
that the label $y$ and spurious feature $a$ are correlated for the majority group in the training data and predicts $y$ using the spurious features. The model only needs to memorize the
points in the minority group. Conversely, if the core features are much nosier, then a model that predicts $y$ via the core features needs to memorize a large fraction of the training data. Due to the inductive bias that seeks to minimize the number of points memorized, the training procedure will select the model that uses spurious features rather than core features to make its predictions.
\begin{wrapfigure}{r}{0.6\textwidth}
   \vspace{-0.15in}
    \includegraphics[width=3.5in]{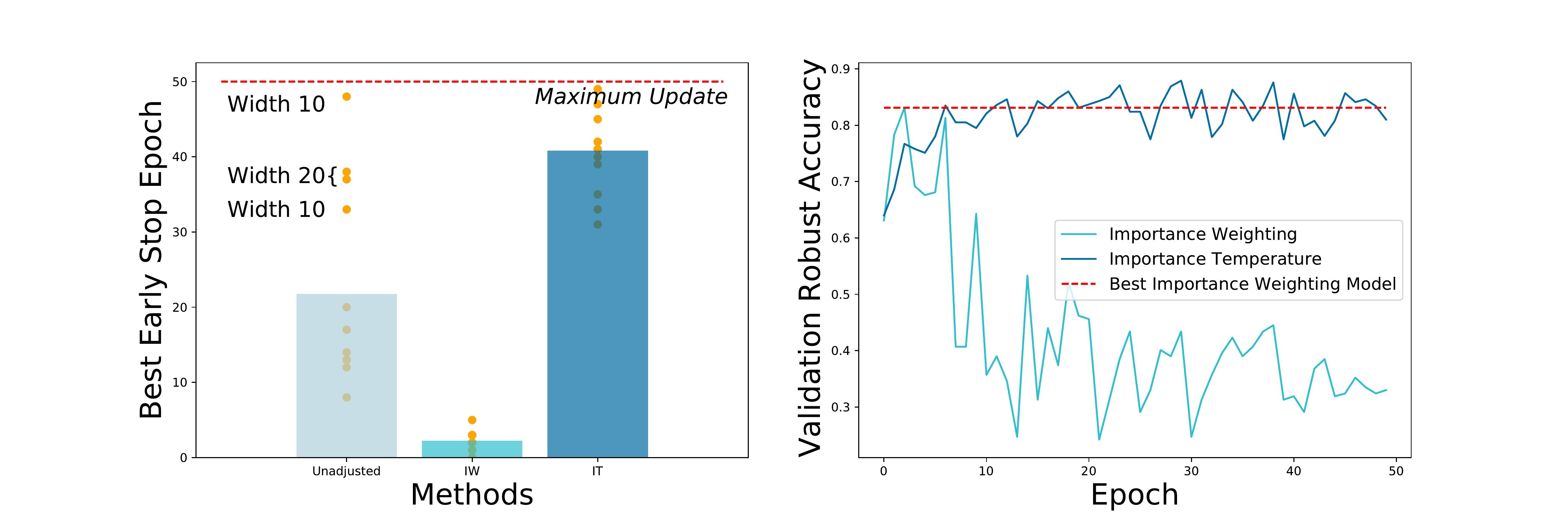}
    \vspace{-0.2in}
    \caption{The statistics of the best epoch for CelebA evaluated by robust validation accuracy. Both ERM and IW need early stopping to add strong regularization, while training longer helps IT generalize better.}
    \label{fig:es}
    \vspace{-0.2in}
\end{wrapfigure}
Importance tempering can help this situation by changing the cost of memorizing data from the different groups. Using importance tempering, we can make the margin requirement on the minority data larger. This makes memorizing a single minority datum more challenging. Concretely, we increase the classifier norm by a larger amount in order to memorize the minority points. In this case, although the \emph{number} of data to be memorized is smaller for the model using spurious features, the \emph{cost} of memorizing the minority data is larger. The inductive bias of the training procedure will then force the model to learn patterns that generalize to both the minority and majority classes, rather than just memorizing the minority.
\begin{figure}[h]
	\centering
	\vspace{-0.1in}
	\subfigure[Optimal inverse temperature $\lambda$ depends on the information ratio. When the information stored in the core feature increases, the optimal inverse temperature $\lambda$ decreases.]{
		\begin{minipage}[b]{0.45\textwidth}
			\includegraphics[width=1\textwidth]{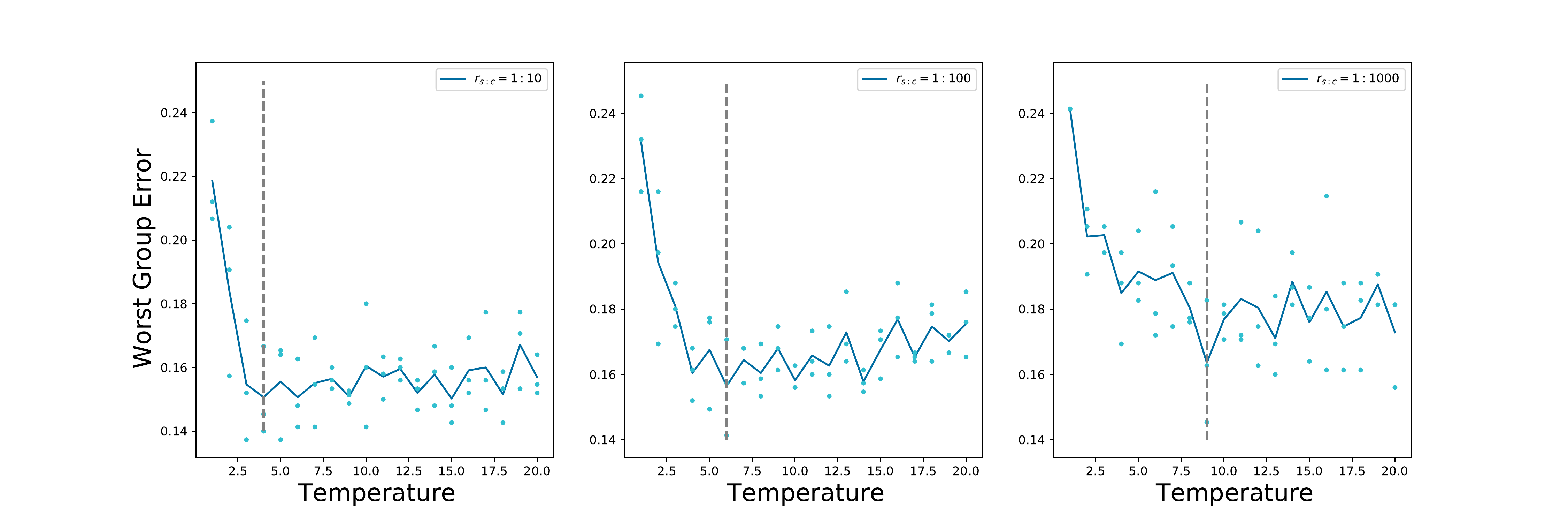}
		\end{minipage}
		\label{optimallambda1}
	}
    	\subfigure[Optimal inverse temperature $\lambda$ depends on the task difficulty. When fitting using core features becomes easier, the optimal inverse temperature $\lambda$ also decreases.]{
    		\begin{minipage}[b]{0.45\textwidth}
   		 	\includegraphics[width=1\textwidth]{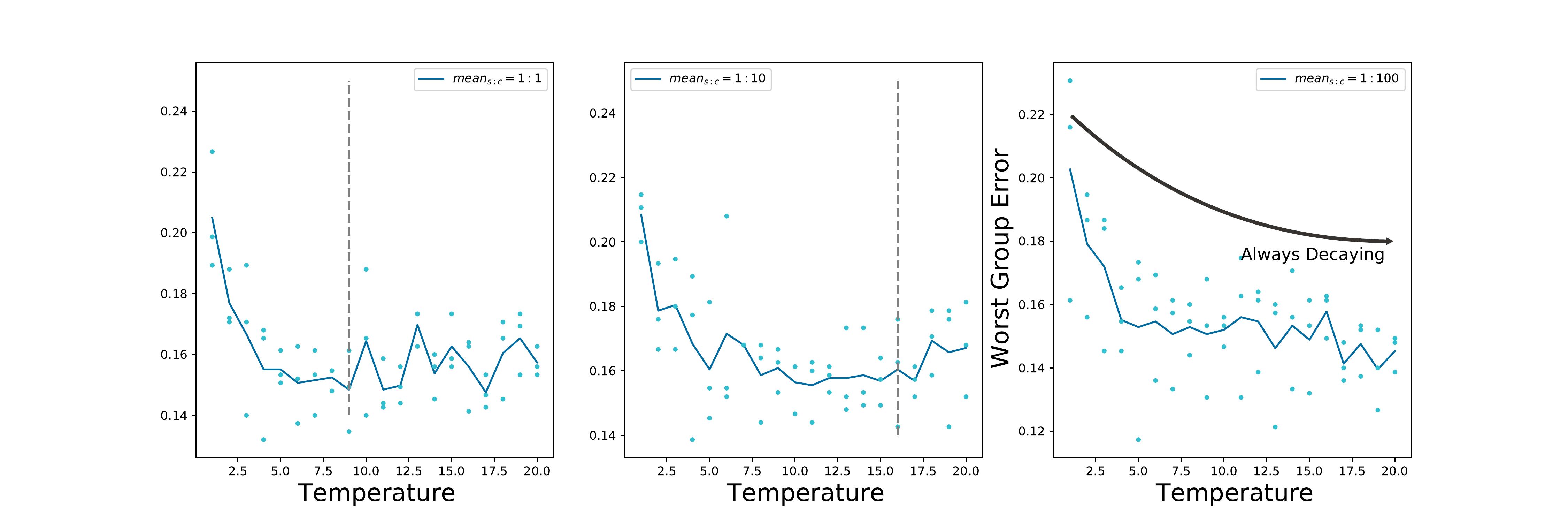}
    		\end{minipage}
		\label{optimallambda2}
    	}
    	\vspace{-0.1in}
	\caption{Illustration of different factors that affect the optimal inverse temperature $\lambda$ setting.}
	\label{fig:optimaltemp}
	\vspace{-0.25in}
\end{figure}

\paragraph{Theory for the example in \cite{sagawa2020investigation}.} To theoretically illustrate the impact of importance tempering and factors that affect the selection of temperature, we revisit a more general version of \cite{sagawa2020investigation}'s example parameterized by more hyper-parameters. In this model, the features $x$ consist of a core feature, a spurious feature, and noise features, \emph{i.e.} $x=[x_c,x_s,x_n]$. For simplicity, \cite{sagawa2020investigation} set the core feature $x_c\in\mathbb{R}$ and the spurious feature $x_s\in\mathbb{R}$ to be scalars. The model can memorize the data through the noisy feature $x_n\in\mathbb{R}^N$. Following \cite{sagawa2020investigation}, we consider a more general dataset
\[\small
\begin{aligned}
x_c|y\sim \mathcal{N}(\mu_cy,(\mu_c\sigma_c)^2),x_s|a\sim\mathcal{N}(\mu_sa,(\mu_s\sigma_c)^2), x_n\sim \mathcal{N}\left(0,\frac{\sigma_n^2}{N}I_N\right)
\end{aligned}
\]
where $\sigma_c,\sigma_s,\sigma_n,\mu_1,\mu_2$ are five constants. $\mu_1,\mu_2$ denote the scale of the features. When the features are larger, the classifier needs a smaller norm to achieve a margin of a fixed size. Due to the inductive bias of training overparameterized models, this task is easier to learn. $\sigma_c,\sigma_s$ denote the noise in the features. Smaller noise means the feature contains more information, \emph{i.e.} a smaller fraction of the data needs to be memorized when this feature is used. For simplicity, we set $\sigma_s=0$. (In \cite{sagawa2020investigation}, $\sigma_s$ is set to be very small, and reducing the noise level on the spurious feature should only make our task harder.) We set $N\gg n$ so that a linear classifier can interpolate and memorize all the data via the noisy feature. The training data is composed of four groups, each corresponding to a combination of the label $y\in\{-1,1\}$ and the spurious attribute $a\in\{-1,1\}$. Each of the two majority groups
with $a = y$ consists of $\frac{n_{\text{maj}}}{2}$ data points $\{(x_{\text{maj}}^{(i)},y_{\text{maj}}^{(i)})\}_{i=1}^n$, and each of two minority groups with $a = -y$ consists of $\frac{n_{\text{min}}}{2}$ data points $\{(x_{\text{min}}^{(i)},y_{\text{min}}^{(i)})\}_{i=1}^n$. We consider linear classifiers with a large margin requirement for the minority class:
\begin{equation}
\small
    \begin{aligned}
    \mathcal{F}_{\text{interpolate}}^\lambda:=\{w:&y_{\text{maj}}^{(i)}(w\cdot x_{\text{maj}}^{(i)})\ge 1,i=1,\cdots,n_{\text{maj}} \text{ and } y_{\text{min}}^{(i)}(w\cdot x_{\text{min}}^{(i)})\ge \lambda,i=1,\cdots,n_{\text{min}}\}.
    \end{aligned}
\end{equation}
\begin{theorem}[Informal]\label{theorem:sagawaexample}
Suppose that $\sigma_n$ is not too large (so memorizing points is expensive) and $\sigma_c$ is also not too large (so that the core feature is reasonably informative). Then there exists a selection of inverse temperature $\lambda$ for the minority group and an estimator $\core=[w_c^c,w_s^c,w_n^c]\in\mathcal{F}_{\text{interpolate}}$ with $\core_s=0$ such that for $\forall \spu=[w_c^s,w_s^s,w_n^s]\in\mathcal{F}_{\text{interpolate}}$ with $\spu_c=0$ we have $\|\core\|\le\|\spu\|$.
\end{theorem}
\vspace{-0.1in}
The proof of this theorem is shown in Appendix \ref{appendix:sagawanorm}, and the discussion of how different factors affect the selection of the temperature can be found in Remark~\ref{remark:selection}. In short, if the core problem is easier and more information is stored in the core feature, a smaller inverse temperature $\lambda$ can be used. To verify our theory, we also perform an experiment on the dataset (\ref{eq:dataset}) with logistic regression on the ReLU random features $\text{ReLU}(Wx)\in \mathbb{R^m}$ \cite{mei2019generalization,montanari2019generalization} and summarize the results in Figure~\ref{fig:optimaltemp}. Figure~\ref{fig:optimaltemp}(a) shows that the optimal temperature increases when there is more information in the spurious feature, while Figure~\ref{fig:optimaltemp}(b) suggests that the optimal temperature increases when the spurious task is easier. Both our theory and empirical experiments show that the importance tempering should be tuned manually, rather than simply setting it to only depend on the imbalance ratio.

Last but not least, we investigate the same dataset that Theorem 1 \cite{sagawa2020investigation} considers (with a special selection of hyper-parameters). In \cite{sagawa2020investigation}, ERM and importance weighting have worst group error larger than $\frac{2}{3}$. We show in Theorem \ref{theorem:betterthanrandom} that, using IT, we can achieve better than random classification results for \emph{all} groups. The proof is presented in Appendix~\ref{appendix:betterthanrandom}.
\begin{theorem}[Informal]\label{theorem:betterthanrandom} Using IT, the inverse temperature $\lambda$ can be selected so that the resulting classifier achieves strictly better than random (\emph{i.e.}, less than error $1/2$) worst-group performance on \cite{sagawa2020investigation}'s example, while ERM and importance weighting cannot.
\end{theorem}

\vspace{-0.1in}
\section{Discussion and Future Works}
\vspace{-0.1in}
We introduce importance tempering, a method that not only improves the 
decision boundary of overparameterized models even when trained on imbalanced data, but also guarantees uniformly good performance over all subgroups of the data both theoretically and empirically. We also observed that the selection of optimal temperature can be different from the optimal importance weight in the label shift setting. We characterized the last layer representation geometry resulting from different ways of incorporating importance tempering. Lastly, in the case of avoiding learning spurious correlations, we found that just considering the imbalance ratio is insufficient to decide the optimal temperature for preventing a model from learning spurious correlations.

In this paper, we have mainly considered the classification problem. It remains an open problem to modify the inductive bias for regression models in order to conquer imbalanced training sets. In addition, our results currently suggest that the importance temperature should be tuned manually. Automatic selection of the temperature is another avenue for future research. At the same time, the optimization process is discussed in this paper, \emph{i.e.} we have only considered the geometric skew in \cite{nagarajan2020understanding} but not the statistical skew. It is interesting to consider the design of optimization methods for our objectives.

\bibliographystyle{unsrt}
\bibliography{it}

\begin{thebibliography}{10}

\bibitem{sagawa2019distributionally}
Shiori Sagawa, Pang~Wei Koh, Tatsunori~B Hashimoto, and Percy Liang.
\newblock Distributionally robust neural networks for group shifts: On the
  importance of regularization for worst-case generalization.
\newblock {\em arXiv preprint arXiv:1911.08731}, 2019.

\bibitem{sagawa2020investigation}
Shiori Sagawa, Aditi Raghunathan, Pang~Wei Koh, and Percy Liang.
\newblock An investigation of why overparameterization exacerbates spurious
  correlations.
\newblock In {\em International Conference on Machine Learning}, pages
  8346--8356. PMLR, 2020.

\bibitem{cao2019learning}
Kaidi Cao, Colin Wei, Adrien Gaidon, Nikos Arechiga, and Tengyu Ma.
\newblock Learning imbalanced datasets with label-distribution-aware margin
  loss.
\newblock {\em arXiv preprint arXiv:1906.07413}, 2019.

\bibitem{torralba2011unbiased}
Antonio Torralba and Alexei~A Efros.
\newblock Unbiased look at dataset bias.
\newblock In {\em CVPR 2011}, pages 1521--1528. IEEE, 2011.

\bibitem{buolamwini2018gender}
Joy Buolamwini and Timnit Gebru.
\newblock Gender shades: Intersectional accuracy disparities in commercial
  gender classification.
\newblock In {\em Conference on fairness, accountability and transparency},
  pages 77--91. PMLR, 2018.

\bibitem{zech2018variable}
John~R Zech, Marcus~A Badgeley, Manway Liu, Anthony~B Costa, Joseph~J Titano,
  and Eric~Karl Oermann.
\newblock Variable generalization performance of a deep learning model to
  detect pneumonia in chest radiographs: a cross-sectional study.
\newblock {\em PLoS medicine}, 15(11):e1002683, 2018.

\bibitem{geirhos2020shortcut}
Robert Geirhos, J{\"o}rn-Henrik Jacobsen, Claudio Michaelis, Richard Zemel,
  Wieland Brendel, Matthias Bethge, and Felix~A Wichmann.
\newblock Shortcut learning in deep neural networks.
\newblock {\em Nature Machine Intelligence}, 2(11):665--673, 2020.

\bibitem{xiao2020noise}
Kai Xiao, Logan Engstrom, Andrew Ilyas, and Aleksander Madry.
\newblock Noise or signal: The role of image backgrounds in object recognition.
\newblock {\em arXiv preprint arXiv:2006.09994}, 2020.

\bibitem{byrd2019effect}
Jonathon Byrd and Zachary Lipton.
\newblock What is the effect of importance weighting in deep learning?
\newblock In {\em International Conference on Machine Learning}, pages
  872--881. PMLR, 2019.

\bibitem{shimodaira2000improving}
Hidetoshi Shimodaira.
\newblock Improving predictive inference under covariate shift by weighting the
  log-likelihood function.
\newblock {\em Journal of statistical planning and inference}, 90(2):227--244,
  2000.

\bibitem{sugiyama2008direct}
Masashi Sugiyama, Taiji Suzuki, Shinichi Nakajima, Hisashi Kashima, Paul von
  B{\"u}nau, and Motoaki Kawanabe.
\newblock Direct importance estimation for covariate shift adaptation.
\newblock {\em Annals of the Institute of Statistical Mathematics},
  60(4):699--746, 2008.

\bibitem{cortes2010learning}
Corinna Cortes, Yishay Mansour, and Mehryar Mohri.
\newblock Learning bounds for importance weighting.
\newblock In {\em Nips}, volume~10, pages 442--450. Citeseer, 2010.

\bibitem{xu2021understanding}
Da~Xu, Yuting Ye, and Chuanwei Ruan.
\newblock Understanding the role of importance weighting for deep learning.
\newblock {\em arXiv preprint arXiv:2103.15209}, 2021.

\bibitem{anonymous2022is}
Ke~Alexander Wang, Niladri~S Chatterji, Saminul Haque, and Tatsunori Hashimoto.
\newblock Is importance weighting incompatible with interpolating classifiers?
\newblock 2021.

\bibitem{anonymous2022stochastic}
Anonymous.
\newblock Stochastic reweighted gradient descent.
\newblock In {\em Submitted to The Tenth International Conference on Learning
  Representations}, 2022.
\newblock under review.

\bibitem{zhai2022understanding}
Runtian Zhai, Chen Dan, Zico Kolter, and Pradeep Ravikumar.
\newblock Understanding why generalized reweighting does not improve over erm,
  2022.

\bibitem{fang2020rethinking}
Tongtong Fang, Nan Lu, Gang Niu, and Masashi Sugiyama.
\newblock Rethinking importance weighting for deep learning under distribution
  shift.
\newblock {\em arXiv preprint arXiv:2006.04662}, 2020.

\bibitem{belkin2019reconciling}
Mikhail Belkin, Daniel Hsu, Siyuan Ma, and Soumik Mandal.
\newblock Reconciling modern machine-learning practice and the classical
  bias--variance trade-off.
\newblock {\em Proceedings of the National Academy of Sciences},
  116(32):15849--15854, 2019.

\bibitem{belkin2021fit}
Mikhail Belkin.
\newblock Fit without fear: remarkable mathematical phenomena of deep learning
  through the prism of interpolation.
\newblock {\em arXiv preprint arXiv:2105.14368}, 2021.

\bibitem{soudry2018implicit}
Daniel Soudry, Elad Hoffer, Mor~Shpigel Nacson, Suriya Gunasekar, and Nathan
  Srebro.
\newblock The implicit bias of gradient descent on separable data.
\newblock {\em The Journal of Machine Learning Research}, 19(1):2822--2878,
  2018.

\bibitem{nacson2019convergence}
Mor~Shpigel Nacson, Jason Lee, Suriya Gunasekar, Pedro Henrique~Pamplona
  Savarese, Nathan Srebro, and Daniel Soudry.
\newblock Convergence of gradient descent on separable data.
\newblock In {\em The 22nd International Conference on Artificial Intelligence
  and Statistics}, pages 3420--3428. PMLR, 2019.

\bibitem{lyu2019gradient}
Kaifeng Lyu and Jian Li.
\newblock Gradient descent maximizes the margin of homogeneous neural networks.
\newblock {\em arXiv preprint arXiv:1906.05890}, 2019.

\bibitem{chizat2020implicit}
Lenaic Chizat and Francis Bach.
\newblock Implicit bias of gradient descent for wide two-layer neural networks
  trained with the logistic loss.
\newblock In {\em Conference on Learning Theory}, pages 1305--1338. PMLR, 2020.

\bibitem{ji2020directional}
Ziwei Ji and Matus Telgarsky.
\newblock Directional convergence and alignment in deep learning.
\newblock {\em arXiv preprint arXiv:2006.06657}, 2020.

\bibitem{li2018algorithmic}
Yuanzhi Li, Tengyu Ma, and Hongyang Zhang.
\newblock Algorithmic regularization in over-parameterized matrix sensing and
  neural networks with quadratic activations.
\newblock In {\em Conference On Learning Theory}, pages 2--47. PMLR, 2018.

\bibitem{vaswani2020each}
Sharan Vaswani, Reza Babanezhad, Jose Gallego, Aaron Mishkin, Simon
  Lacoste-Julien, and Nicolas~Le Roux.
\newblock To each optimizer a norm, to each norm its generalization.
\newblock {\em arXiv preprint arXiv:2006.06821}, 2020.

\bibitem{ji2020gradient}
Ziwei Ji, Miroslav Dud{\'\i}k, Robert~E Schapire, and Matus Telgarsky.
\newblock Gradient descent follows the regularization path for general losses.
\newblock In {\em Conference on Learning Theory}, pages 2109--2136. PMLR, 2020.

\bibitem{lin2017focal}
Tsung-Yi Lin, Priya Goyal, Ross Girshick, Kaiming He, and Piotr Doll{\'a}r.
\newblock Focal loss for dense object detection.
\newblock In {\em Proceedings of the IEEE international conference on computer
  vision}, pages 2980--2988, 2017.

\bibitem{kini2021label}
Ganesh~Ramachandra Kini, Orestis Paraskevas, Samet Oymak, and Christos
  Thrampoulidis.
\newblock Label-imbalanced and group-sensitive classification under
  overparameterization.
\newblock {\em arXiv preprint arXiv:2103.01550}, 2021.

\bibitem{ye2020identifying}
Han-Jia Ye, Hong-You Chen, De-Chuan Zhan, and Wei-Lun Chao.
\newblock Identifying and compensating for feature deviation in imbalanced deep
  learning.
\newblock {\em arXiv preprint arXiv:2001.01385}, 2020.

\bibitem{narasimhan2021training}
Harikrishna Narasimhan and Aditya~Krishna Menon.
\newblock Training over-parameterized models with non-decomposable objectives.
\newblock {\em arXiv preprint arXiv:2107.04641}, 2021.

\bibitem{menon2020long}
Aditya~Krishna Menon, Sadeep Jayasumana, Ankit~Singh Rawat, Himanshu Jain,
  Andreas Veit, and Sanjiv Kumar.
\newblock Long-tail learning via logit adjustment.
\newblock {\em arXiv preprint arXiv:2007.07314}, 2020.

\bibitem{wang2021seesaw}
Jiaqi Wang, Wenwei Zhang, Yuhang Zang, Yuhang Cao, Jiangmiao Pang, Tao Gong,
  Kai Chen, Ziwei Liu, Chen~Change Loy, and Dahua Lin.
\newblock Seesaw loss for long-tailed instance segmentation.
\newblock In {\em Proceedings of the IEEE/CVF Conference on Computer Vision and
  Pattern Recognition}, pages 9695--9704, 2021.

\bibitem{anonymous2022learning}
Anonymous.
\newblock Learning towards the largest margins.
\newblock In {\em Submitted to The Tenth International Conference on Learning
  Representations}, 2022.
\newblock under review.

\bibitem{papyan2020prevalence}
Vardan Papyan, XY~Han, and David~L Donoho.
\newblock Prevalence of neural collapse during the terminal phase of deep
  learning training.
\newblock {\em Proceedings of the National Academy of Sciences},
  117(40):24652--24663, 2020.

\bibitem{fang2021layer}
C.~Fang, H.~He, Q.~Long, and W.~Su.
\newblock Exploring deep neural networks via layer-peeled model: {M}inority
  collapse in imbalanced training.
\newblock {\em Proceedings of the National Academy of Sciences (in press)},
  2021.

\bibitem{guo2017one}
Yandong Guo and Lei Zhang.
\newblock One-shot face recognition by promoting underrepresented classes.
\newblock {\em arXiv preprint arXiv:1707.05574}, 2017.

\bibitem{khan2019striking}
Salman Khan, Munawar Hayat, Syed~Waqas Zamir, Jianbing Shen, and Ling Shao.
\newblock Striking the right balance with uncertainty.
\newblock In {\em Proceedings of the IEEE/CVF Conference on Computer Vision and
  Pattern Recognition}, pages 103--112, 2019.

\bibitem{shawe1999optimizing}
Grigoris Karakoulas~John Shawe-Taylor and Grigoris Karakoulas.
\newblock Optimizing classifiers for imbalanced training sets.
\newblock {\em Advances in neural information processing systems}, 11(11):253,
  1999.

\bibitem{fumera2002cost}
Giorgio Fumera and Fabio Roli.
\newblock Cost-sensitive learning in support vector machines.
\newblock {\em VIII Convegno Associazione Italiana per L’Intelligenza
  Artificiale}, 2002.

\bibitem{storkey2009training}
Amos Storkey.
\newblock When training and test sets are different: characterizing learning
  transfer.
\newblock {\em Dataset shift in machine learning}, 30:3--28, 2009.

\bibitem{han2018co}
Bo~Han, Quanming Yao, Xingrui Yu, Gang Niu, Miao Xu, Weihua Hu, Ivor Tsang, and
  Masashi Sugiyama.
\newblock Co-teaching: Robust training of deep neural networks with extremely
  noisy labels.
\newblock {\em arXiv preprint arXiv:1804.06872}, 2018.

\bibitem{garg2020unified}
Saurabh Garg, Yifan Wu, Sivaraman Balakrishnan, and Zachary~C Lipton.
\newblock A unified view of label shift estimation.
\newblock {\em arXiv preprint arXiv:2003.07554}, 2020.

\bibitem{youbi2021simple}
Badr Youbi~Idrissi, Martin Arjovsky, Mohammad Pezeshki, and David Lopez-Paz.
\newblock Simple data balancing achieves competitive worst-group-accuracy.
\newblock {\em arXiv e-prints}, pages arXiv--2110, 2021.

\bibitem{krizhevsky2009learning}
Alex Krizhevsky, Geoffrey Hinton, et~al.
\newblock Learning multiple layers of features from tiny images.
\newblock 2009.

\bibitem{xiao2017fashion}
Han Xiao, Kashif Rasul, and Roland Vollgraf.
\newblock Fashion-mnist: a novel image dataset for benchmarking machine
  learning algorithms.
\newblock {\em arXiv preprint arXiv:1708.07747}, 2017.

\bibitem{he2016deep}
Kaiming He, Xiangyu Zhang, Shaoqing Ren, and Jian Sun.
\newblock Deep residual learning for image recognition.
\newblock In {\em Proceedings of the IEEE conference on computer vision and
  pattern recognition}, pages 770--778, 2016.

\bibitem{zhu2021geometric}
Zhihui Zhu, Tianyu Ding, Jinxin Zhou, Xiao Li, Chong You, Jeremias Sulam, and
  Qing Qu.
\newblock A geometric analysis of neural collapse with unconstrained features.
\newblock {\em arXiv preprint arXiv:2105.02375}, 2021.

\bibitem{ji2021unconstrained}
Wenlong Ji, Yiping Lu, Yiliang Zhang, Zhun Deng, and Weijie~J. Su.
\newblock An unconstrained layer-peeled perspective on neural collapse, 2021.

\bibitem{buda2018systematic}
Mateusz Buda, Atsuto Maki, and Maciej~A Mazurowski.
\newblock A systematic study of the class imbalance problem in convolutional
  neural networks.
\newblock {\em Neural Networks}, 106:249--259, 2018.

\bibitem{liu2021just}
Evan~Z Liu, Behzad Haghgoo, Annie~S Chen, Aditi Raghunathan, Pang~Wei Koh,
  Shiori Sagawa, Percy Liang, and Chelsea Finn.
\newblock Just train twice: Improving group robustness without training group
  information.
\newblock In {\em International Conference on Machine Learning}, pages
  6781--6792. PMLR, 2021.

\bibitem{arjovsky2019invariant}
Martin Arjovsky, L{\'e}on Bottou, Ishaan Gulrajani, and David Lopez-Paz.
\newblock Invariant risk minimization.
\newblock {\em arXiv preprint arXiv:1907.02893}, 2019.

\bibitem{devlin2018bert}
Jacob Devlin, Ming-Wei Chang, Kenton Lee, and Kristina Toutanova.
\newblock Bert: Pre-training of deep bidirectional transformers for language
  understanding.
\newblock {\em arXiv preprint arXiv:1810.04805}, 2018.

\bibitem{gururangan2018annotation}
Suchin Gururangan, Swabha Swayamdipta, Omer Levy, Roy Schwartz, Samuel~R
  Bowman, and Noah~A Smith.
\newblock Annotation artifacts in natural language inference data.
\newblock {\em arXiv preprint arXiv:1803.02324}, 2018.

\bibitem{opper1995statistical}
Manfred Opper.
\newblock Statistical mechanics of learning: Generalization.
\newblock {\em The handbook of brain theory and neural networks}, pages
  922--925, 1995.

\bibitem{mei2019generalization}
Song Mei and Andrea Montanari.
\newblock The generalization error of random features regression: Precise
  asymptotics and the double descent curve.
\newblock {\em Communications on Pure and Applied Mathematics}, 2019.

\bibitem{montanari2019generalization}
Andrea Montanari, Feng Ruan, Youngtak Sohn, and Jun Yan.
\newblock The generalization error of max-margin linear classifiers:
  High-dimensional asymptotics in the overparametrized regime.
\newblock {\em arXiv preprint arXiv:1911.01544}, 2019.

\bibitem{nakkiran2019deep}
Preetum Nakkiran, Gal Kaplun, Yamini Bansal, Tristan Yang, Boaz Barak, and Ilya
  Sutskever.
\newblock Deep double descent: Where bigger models and more data hurt.
\newblock {\em arXiv preprint arXiv:1912.02292}, 2019.

\bibitem{zagoruyko2016wide}
Sergey Zagoruyko and Nikos Komodakis.
\newblock Wide residual networks.
\newblock {\em arXiv preprint arXiv:1605.07146}, 2016.

\bibitem{nagarajan2020understanding}
Vaishnavh Nagarajan, Anders Andreassen, and Behnam Neyshabur.
\newblock Understanding the failure modes of out-of-distribution
  generalization.
\newblock {\em arXiv preprint arXiv:2010.15775}, 2020.

\bibitem{davis2020stochastic}
Damek Davis, Dmitriy Drusvyatskiy, Sham Kakade, and Jason~D Lee.
\newblock Stochastic subgradient method converges on tame functions.
\newblock {\em Foundations of computational mathematics}, 20(1):119--154, 2020.

\bibitem{bach2017breaking}
Francis Bach.
\newblock Breaking the curse of dimensionality with convex neural networks.
\newblock {\em The Journal of Machine Learning Research}, 18(1):629--681, 2017.

\bibitem{kakade2008complexity}
Sham~M Kakade, Karthik Sridharan, and Ambuj Tewari.
\newblock On the complexity of linear prediction: Risk bounds, margin bounds,
  and regularization.
\newblock 2008.

\bibitem{weinan2019barron}
E~Weinan, Chao Ma, and Lei Wu.
\newblock Barron spaces and the compositional function spaces for neural
  network models.
\newblock {\em arXiv preprint arXiv:1906.08039}, 2019.

\bibitem{wolf2019huggingface}
Thomas Wolf, Lysandre Debut, Victor Sanh, Julien Chaumond, Clement Delangue,
  Anthony Moi, Pierric Cistac, Tim Rault, R{\'e}mi Louf, Morgan Funtowicz,
  et~al.
\newblock Huggingface's transformers: State-of-the-art natural language
  processing.
\newblock {\em arXiv preprint arXiv:1910.03771}, 2019.

\end{thebibliography}

\newpage


\appendix

\newpage
\section{Connection to hard-margin Support Vector Machine}
In this section, we will adopt the results in \cite{lyu2019gradient} and \cite{ji2021unconstrained} to show the impact of importance tempering on the convergence direction. Before starting formal discussion, we first introduce the regularity assumption on our model

\begin{assumption}[Regularity]
\label{asm: regularity}
    $q(x,\cdot)$ is locally Lipschitz and admits a chain rule for any fixed $x$.
\end{assumption}
This is a technical assumption on the network output, as shown in \cite{davis2020stochastic,lyu2019gradient}, the output of almost every neural network satisfies the regularity condition (as long as the neural network
is composed by definable pieces in an o-minimal structure, e.g., ReLU, sigmoid, LeakyReLU). Then we introduce the formal version of Theorem \ref{thm: implicit bias}
\begin{theorem}
\label{thm: implicit bias formal}
Suppose Assumption \ref{asm: regularity} and \ref{asm: homogeneity} holds for $q(x,\theta)$. Let $\theta(t)$ denote the model parameters trained with gradient flow at time $t$. If there exists a time $t_0$ such that $\mathcal{L}^{\text{IT}}(\theta(t_0))<\frac{1}{n}$, then any limit point of $\frac{\theta(t_0)}{\|\theta(t_0)\|}$ is along the direction (i.e., a scalar multiple of) of a Karush-Kuhn-Tucker (KKT) point of the following minimum-norm separation problem:
$$
\min_w \|w\| \: \text{ s.t. } \:
y_iq(x_i,\theta)\ge 1/f[g_i], \hspace{.1in} i = 1,\ldots,n.
$$
\end{theorem}
Its proof is straightforward based on the following result in \cite{lyu2019gradient}:
\begin{theorem}[Theorem 4.4 of \cite{lyu2019gradient}]
\label{thm: lyu}
Denote the loss function as $\mathcal{L}(\theta):=\frac{1}{n}\sum_{i=1}^n\ell(y_i q(x_i,\theta))$, where $\ell(q)=e^{-q}$ denotes the exponential loss, for gradient flow with Assumption \ref{asm: homogeneity} and \ref{asm: regularity} hold, if we further assume that there exists a time $t_0$ such that $\mathcal{L}(\theta)(t_0)<\frac{1}{n}$, then any limit point $\bar{\theta}$of $\{\frac{\theta(t)}{\|\theta(t)\|}:t>0\}$ is along a KKT point of the following constrained optimization problem:
\begin{equation}
    \label{SVM: Lyu}
    \min \quad \frac{1}{2}\|{\theta}\|_{2}^{2} \quad \text { s.t. } \quad y_iq(x_i,\theta) \geq 1 \quad \forall 1\leq i\leq n
\end{equation}
\end{theorem}
\begin{proof}[Proof of Theorem \ref{thm: implicit bias formal}]
The proof of Theorem \ref{thm: implicit bias} simply follows the fact that both the label $y_i$ and group temperature $f[g_i]$ are determined at an instance level, thus we can absorb the group temperature in the label. Note that we have no requirement on the dataset in Theorem \ref{thm: lyu}, which allows us to create a synthetic dataset $\{(x_i,f[g_i]y_i)\}_{i=1}^n$ and apply Theorem \ref{thm: lyu} on this synthetic dataset. In this way, we can conclude that the limit point of the gradient flow is along the direction of a KKT point of the following minimum-norm separation problem:
\begin{equation}
    \min \quad \frac{1}{2}\|{\theta}\|_{2}^{2} \quad \text { s.t. } \quad y_iq(x_i,\theta) \geq 1/f[g_i] \quad \forall 1\leq i\leq n
\end{equation}
as desired.
\end{proof}

\section{The Generalization Theorem}
\label{appendix:generalization}
In this section, we consider the generalization property of a importance tempering large margin two-layer neural network (Theorem \ref{thm:sqrt}). Let us consider a binary classification problem with a training set ${(x_i,y_i)}_{i\in[n]}$ of $n$ pairs of observations with $x_i\in\mathbb{R}^d$ and $y_i\in{-1,1}$. We predict the function using a two-layer neural network
$$
h_m(w,x)=\frac{1}{m}\sum_{j=1}^m\phi(w_j,x),
$$
where $m\ge 1$ is the number of units and $w=(w_j)_{j\in[m]}$ are trainable parameters. We refer to $\phi$ a feature function and in this section we assume $\phi$ is 2-homogeneous. We train the two-layer neural network using importance tempering and finally convergences to the following large margin SVM problem
\begin{equation}
    \begin{aligned}
    \min& \quad \|w\|\\
    \text{subject to }& \quad \gamma_1h_m(w,x_i)y_i\ge 1, \text{ for } y_i=1\\
    &\quad  \gamma_{-1}h_m(w,x_i)y_i\ge 1, \text{ for } y_i=-1
    \end{aligned}
\end{equation}

Following \cite{bach2017breaking,chizat2020implicit}, we characterize the large margin solution of the two-layer neural network utilizing the integral representation and its corresponding variational $\mathcal{F}_1$ norm. We formulate the large margin problem of a infinite wide two-layer neural network using the following integral representation
$$
\mathcal{C}:=\max_{\mu\in\mathcal{P}(\mathbb{S}^{p-1})} \min_{i\in[n]} \gamma_{y_i} y_i\int_{\mathbb{S}^{p-1}} \phi(\theta,x_i)d\mu (\theta)
$$

To bound $\mathcal{C}$, we define the complexity of the dataset $S_n=(x_i,y_i)_{i=1}^n$ is formulated as 
$$
\Delta_r(S_n):=\sup_P\{\inf_{y_i\not=y_i'}\|P(x_{i})-P(x_{i'})\|: P\text{ is a rank}-r \text{ orthogonal projection}\}
$$ 

\begin{lemma}\label{lemma:boundongamma} Assume that $\|x_i\|\le R$ for $i\in [n]$. For any $\epsilon\in(0,1)$ and $r\in[d]$, there exists $C(r),C_\epsilon(r)>0$ such that
\begin{equation}
    \begin{aligned}
    \mathcal{C}\ge \min_{r\in[d]}\min\left\{C(r),C_\epsilon(r)\left(\frac{\Delta_r(S_n)}{R}\right)^{\frac{r+s}{2-\epsilon}}(\gamma_1+\gamma_2)^{\frac{2d+3}{2-\epsilon}}\right\}
    \end{aligned}
    \label{eq:datasetdelta}
\end{equation}

\end{lemma}

\begin{proof} Let $\text{dist}_{\mathcal{S}}$ be the distance function to a set $\mathcal{S}$, \emph{i.e.} $\text{dist}_{\mathcal{S}}(x)=\inf_{y\in\mathcal{S}}\|x-y\|$. We know that function $\text{dist}_{\mathcal{S}}$ is 1-Lipschitz. We denote $D_{\pm}:=\{x_i:y_i=\pm 1\}$ and $P_r$ the projection that achieves the supremum in Equation (\ref{eq:datasetdelta}). Now let us consider the following function
$$
f_r(x)=2\max\left(0,\frac{1}{\gamma_1}-\frac{2\text{dist}_{P_r(D_+)}(P_r(x))}{(\gamma_1+\gamma_{-1})\Delta_r(S_n)}\right)-2\max\left(0,\frac{1}{\gamma_{-1}}-\frac{2\text{dist}_{P_r(D_+)}(P_r(x))}{(\gamma_1+\gamma_{-1})\Delta_r(S_n)}\right).
$$
This function is $\frac{4}{(\gamma_1+\gamma_{-1})\Delta_r(S_n)}$ Lipschitz, satisfies $\|f\|_{\infty}\le 2$ and $\gamma_{y_i}y_if(x_i)=2$ for all $i\in[n]$. Using the approximation results of Lipschitz function in $\mathcal{F}_{1}$ (Prop 6 and Section 4.5 in \cite{bach2017breaking}), we knows that we have a function 
$$
\|\hat f\|\le O\left( C(\epsilon,r)\left(\frac{\Delta_r(S_n)}{R}\right)\right)^{\frac{d+3}{2-\epsilon}}(\gamma_1+\gamma_{-1})^{\frac{2d+3}{2-\epsilon}}
$$
such that $\sup_{\|x\|\le R}|\hat f(x)-f_r(x)|\le \frac{1}{\gamma_1+\gamma_{-1}}$. Thus we know that $f_r$ is a separation function and the minimum norm solution only haves smaller norm.
\end{proof}

\begin{theorem} \label{thm: it gen bound}
Suppose we have a class-imbalanced binary classification task with $n_1$ positive examples sampled from distribution $p_1$ and $n_{-1} < n_1$ negative examples sampled from distribution $p_{-1}$.  Then if we train a infinite wide two-layer neural network with importance tempering objective function $\exp(-\gamma_{y_i}y_ih_m(w,x))$ on the negative class, with probability at lest $1-\delta$ over the training set, the limiting model have
$$ \max_i \mathbb{P}_{x\sim p_i}\left[y_iq(x,\theta)\le0\right] \lesssim \max_i \gamma_i\sqrt{\frac{1}{n_i}}\left(\frac{R}{\Delta_r(\mathcal{P})}\right)^{\frac{r+3}{2-\epsilon}}(\gamma_1+\gamma_{-1})^{-\frac{2d+3}{2-\epsilon}}+\sqrt{\frac{\log\frac{1}{\delta}+\log\frac{1}{\gamma_i}}{n_i}}. $$
If we fixed $\gamma_1+\gamma_{-1}$ as a constant, the best way to select the temperature will become $\gamma_i\propto\sqrt{n_i}$ to minimize the right hand size function.
\end{theorem}

\begin{proof}
If we train with importance tempering objective function $\exp(-\gamma_{y_i}y_ih_m(w,x))$, then by Theorem~\ref{thm: implicit bias}, we converge to the KKT point of
\begin{equation}
    \begin{aligned}
    \min& \quad \|h\|_{\mathcal{F}_1}\\
    \text{subject to }& \quad \gamma_1h_m(w,x_i)y_i\ge 1, \text{for } y_i=1\\
    &\quad  \gamma_{-1}h_m(w,x_i)y_i\ge 1, \text{for } y_i=-1
    \end{aligned}
\end{equation}
Using Theorem \ref{lemma:boundongamma}, we knows that
$$
\max_{\mu\in\mathcal{P}(\mathbb{S}^{p-1})} \min_{i\in[n]} \gamma_{y_i} y_i\int_{\mathbb{S}^{p-1}} \phi(\theta,x_i)d\mu (\theta)\ge\min_{r\in[d]}\min\left\{C(r),C_\epsilon(r)\left(\frac{\Delta_r(S_n)}{R}\right)^{\frac{r+s}{2-\epsilon}}\right\}
$$
Combined the Rademacher complexity bound in \cite{bach2017breaking} (Prop 7. \cite{bach2017breaking}), we have 
$$
\text{Rad}_n\le \frac{\|f\|_{\mathcal{F}_1}}{\sqrt{n}}\le \frac{1}{\sqrt{n}}\left(\frac{R}{\Delta_r(\mathcal{P})}\right)^{\frac{r+3}{2-\epsilon}}(\gamma_1+\gamma_{-1})^{-\frac{2d+3}{2-\epsilon}}.
$$
Then we can apply the standard margin-based generalization bound (Theorem 2 of \cite{kakade2008complexity}),  to obtain with probability $1-\delta$, we have
\begin{equation}
    \begin{aligned}
    \mathbb{P}_{x\sim p_i}\left[y_iq(x,\theta)\le0\right] &\le 4\gamma_i \text{Rad}_{n_i}+\sqrt{\frac{\log\frac{1}{\delta}+\log\frac{1}{\gamma_i}}{n_i}}\\
    &\le \gamma_i\sqrt{\frac{1}{n_i}}\left(\frac{R}{\Delta_r(\mathcal{P})}\right)^{\frac{r+3}{2-\epsilon}}(\gamma_1+\gamma_{-1})^{-\frac{2d+3}{2-\epsilon}}+\sqrt{\frac{\log\frac{1}{\delta}+\log\frac{1}{\gamma_i}}{n_i}}
    \end{aligned}
\end{equation}
\end{proof}

\begin{remark}
We provided generalization bound for two-layer neural network because one can know the margin for two-layer case using technique in \cite{chizat2020implicit}. We can also generalize our theorem to ResNet. Theorem 9 \cite{weinan2019barron} knows that the margin of ResNet can be bounded by the two-layer neural network. Together with the Rademacher complexity bound (Theorem 12 \cite{weinan2019barron}), we can a similar bound of ResNet. If we can assume the margin after temperature to become $O(1)$, then our theorem can also be applied to general classifiers. 
\end{remark}

\section{Proof for ULPM with importance tempering}

In this section, we present the proof of Theorem \ref{theorem:collapse} in Section \ref{section:multiclass}. We mainly follows the unconstrained layer-peeled model \cite{ji2021unconstrained}, a top-down model to understand how overparameterized well-trained end-to-end deep architectures can effectively extract features. We aim to show the last layer feature will behave very different geometric properties under the extremely imbalanced setting \cite{fang2021layer}.

\subsection{Vanilla Cross-entropy Objective Leads To Minority Collapse Solution}   

We first consider the vanilla cross-entropy loss. We will show that the minority classes are distinguishable in terms of their last layer features.  Following \cite{fang2021layer,ji2021unconstrained}, we consider the unconstrained layer-peeled model (ULPM) temperature:
	\begin{equation}
	\begin{aligned}
	&\min_{\W,\Hb} \mathcal{L}(\W,\Hb) \\
	:= &\min_{\W,\Hb} -\sum_{k=1}^{K}\sum_{i=1}^{n_k} \log\left( \frac{\exp(\w_k^\top \h_{k,i})}{\sum_{j=1}^K\exp(\w_j^\top \h_{k,i})} \right).    
	\end{aligned}
    \end{equation}
 
 \cite{lyu2019gradient,ji2020gradient} proved that gradient descent on this loss will converge to the solution of the minimum-norm separation problem
 \begin{equation}
\begin{aligned}
\label{SVM:appendix}
&\min _{\W, \Hb} \frac{1}{2}||\W||_F^2+\frac{1}{2}||\Hb||_F^2\\
s.t.& \w_k^\top \h_{k,i}-\w_j^\top \h_{k,i} \geq 1,\quad  k\not=j\in[K],i\in [n_k].
\end{aligned}
\end{equation}

We first prove that the within-class variation of the activation
becomes negligible as these activation collapse to their class mean, \emph{i.e.} $h_{k,i_1}=h_{k,i_2}=\frac{1}{n_k}\sum_{i=1}^{n_k}h_{k,i}$ for all $k\in[K]$. If we have a feasible solution $(W,H)$ subject to $\|h_{k,i_1}-h_{k,i_2}\|\ge\epsilon>0$ for some $k\in[K],i_1,i_2\in [k]$. We can construct $\tilde{W}, \tilde{H}$ via letting $\tilde{h}_{k,i_1}=\tilde{h}_{k,i_2}=\frac{1}{2}(h_{k,i_1}+h_{k,i_2})$ and all the other vectors unchanged $\tilde{h}_{k^\prime,i^\prime}=h_{k^\prime,i^\prime}, \tilde{w}_{k^\prime}={w}_{k^\prime}$ for all $(k^\prime,i^\prime)\not = (k,i_1)$ or $(k,i_1)$. We first check that $\tilde{W}, \tilde{H}$ is also a feasible solution for
$$
w_k^\top \tilde{h}_{k,i_1}-w_j^\top \tilde{h}_{k,i_1}=w_k^\top \tilde{h}_{k,i_2}-w_j^\top \tilde{h}_{k,i_2}=\frac{1}{2}\left[(\w_k^\top \h_{k,i_1}-\w_j^\top \h_{k,i_2})+(\w_k^\top \h_{k,i_1}-\w_j^\top \h_{k,i_2})\right]\ge1.
$$

At the same time, the objective function will decay at least $\epsilon^2$ for we have $\frac{1}{2}[\|\tilde{h}_{k,i_1}\|^2+\|\tilde{h}_{k,i_2}\|^2]=\frac{1}{2}[\|{h}_{k,i_1}\|^2+\|{h}_{k,i_2}\|^2-2\|{h}_{k,i_1}-{h}_{k,i_2}\|^2]\le\frac{1}{2}[\|{h}_{k,i_1}\|^2+\|{h}_{k,i_2}\|^2]-\epsilon^2$. Thus we know that the within-class variation of the activation
becomes negligible as these activation collapse to their class mean, \emph{i.e.} $h_{k,i_1}=h_{k,i_2}=\frac{1}{n_k}\sum_{i=1}^{n_k}h_{k,i}$ for all $k\in[K]$. Thus problem (\ref{SVM:appendix}) can be formulated as
 \begin{equation}
\begin{aligned}
&\min _{\W, \Hb} \frac{1}{2}||\W||_F^2+\frac{1}{2}\sum_{i=1}^K n_K||h_k||^2\\
s.t.& \w_k^\top \h_{k}-\w_j^\top \h_{k} \geq 1,\quad  k\not=j\in[K].
\end{aligned}
\end{equation}
To balance the problem, we further consider another reparameterization of $W$. We substitute $W$ by $\sqrt{Kn/2}W$, where $n=\sum_{i=1}^K n_i$ is the total number of data, and lead the following problem
\begin{equation}
\begin{aligned}
\label{SVM:reweighted}
&\min _{\W, \Hb} \frac{1}{2}||\W||_F^2+\frac{1}{2}\sum_{i=1}^K \frac{2n_k}{Kn}||h_k||^2\\
s.t.& \w_k^\top \h_{k}-\w_j^\top \h_{k} \geq 1,\quad  k\not=j\in[K].
\end{aligned}
\end{equation}

The newly introduced parameterization will only changes the the scale of the solution but will not change the angle between them. Thus this reformulation will not change our final conclusion.

We first proved that on the limiting objective function, \emph{i.e.} the coefficient before the norm of minority's feature vector in objective function limits to zero and leads to objective (\ref{SVM:minority collapse}), will cause a minority collapse solution. For the features in not shown in the objective function, we also move the constraints on the minority to have the proof.

\begin{lemma}
The global optimal solution of the following problem
\begin{equation}
\begin{aligned}
\label{SVM:minority collapse}
&\min _{\W, \Hb} \mathcal{L}_{\lim}(\W,\Hb):=\frac{1}{2}||\W||_F^2+\frac{1}{2}\sum_{i=1}^{K/2} ||h_k||^2\\
s.t.& \w_k^\top \h_{k}-\w_j^\top \h_{k} \geq 1,\quad  \forall k\not=j, k\leq\frac{K}{2},
\end{aligned}
\end{equation}
satisfies the condition that $\w_k=\w_{k'},\forall \frac{K}{2}+1\leq k<k'\leq K$. Moreover, for any feasible solution $(\W,\Hb)$ with $\|\w_k-\w_{k'}\|\geq \epsilon$ for some $\frac{K}{2}+1\leq k<k'\leq K$, we can find another feasible solution $(\W',\Hb')$ such that $\mathcal{L}_{\lim}(\W,\Hb)-\mathcal{L}_{\lim}(\W',\Hb')\geq \epsilon^2$
\end{lemma}   
\begin{proof}
First we observe that the optimal solution $(\W,\Hb)$ must satisfy $\sum_{i=1}^k \w_i = 0$, otherwise we can set $\tilde{\w}_i=\w_i-\frac{1}{K}\sum_{i=1}^k\w_i$ and $\tilde{\W}=(\tilde{\w}_1,\cdots,\tilde{\w}_K)$ such that
$$\tilde{\w}_k^\top \h_{k}-\tilde{\w}_j^\top \h_{k} = \w_k^\top \h_{k}-\w_j^\top \h_{k} \geq 1, \quad \forall 1\leq k\leq K/2, j\neq k,$$
and 
\begin{equation}
    \|\tilde{\W}\|_F^2=\sum_{k=1}^K\|\w_k-\frac{1}{K}\sum_{i=1}^K\w_i\|^2=\|\W\|_F^2-\frac{1}{K}\|\sum_{i=1}^K\w_i\|^2<\|\W\|_F^2,
\end{equation}
which contradicts the optimality of $(\W,\Hb)$.

Second we observe that the optimal solution $(\W,\Hb)$ must satisfy $\w_{\frac{K}{2}+1}=\cdots=\w_K$, otherwise we can set $\hat{\w}_i=\w_i,\forall 1\leq i\leq K/2$, $\hat{\w}_{\frac{K}{2}+1}=\cdots=\hat{\w}_K:=\frac{2}{K}\sum_{i=K/2+1}^{K}w_i$ and $\hat{\W}=(\tilde{\w}_1,\cdots,\tilde{\w}_K)$ such that
$$\hat{\w}_k^\top \h_{k}-\hat{\w}_j^\top \h_{k} = \frac{2}{K}\sum_{j=K/2+1}^K(\w_k^\top \h_{k}-\w_j^\top \h_{k}) \geq 1, \forall 1\leq k\leq K/2, \frac{K}{2}+1 \leq  j\leq K,$$
and by Cauchy-Schwarz inequality:
\begin{equation}
\|\hat{\W}\|_F^2=\sum_{k=1}^{K/2}\|\w_k\|^2+\frac{K}{2}\|\frac{2}{K}\sum_{k=K/2+1}^K\w_k\|^2<\sum_{k=1}^{K/2}\|\w_k\|^2+\sum_{k=K/2+1}^K\|\w_k\|^2<\|\W\|_F^2,
\end{equation}
which contradicts the optimality of $(\W,\Hb)$.

Finally, for any feasible solution $(\W,\Hb)$, if we have $\|\w_k-\w_{k'}\|\geq \epsilon$ for $\frac{K}{2}+1\geq k<k'<K$. Then we can replace $\w_k$ and $\w_{k'}$ by $\frac{1}{2}(\w_k+\w_{k'})$, which is still feasible for \eqref{SVM:minority collapse}, and
\begin{equation}
    \|\w_k\|^2+\|\w_{k'}\|^2-2\|\frac{1}{2}(\w_k+\w_{k'})\|^2=\frac{1}{2}\|\w_k-\w_{k'}\|^2\geq \epsilon^2
\end{equation}
\end{proof}

Then we aim to show that if the global solution $(\Hb^\ast,\W^\ast)$ of (\ref{SVM:or}) have a limit, then the limit is a minority collapse solution, \emph{i.e.}
    $$
    \lim_{R\rightarrow \infty} w_k^\ast-w_{k'}^\ast = 0, \text{ for all } K/2<k<k'\le K.
    $$
Once a solution satisfies the constraints on the majority constraints $\w_k^\top \h_{k}-\w_j^\top \h_{k} \geq 1,\quad  \forall k\not=j, k\leq\frac{K}{2}$, if $\|\w_k-\w_{k'}\|\geq \epsilon$ for some $\frac{K}{2}+1\leq k<k'\leq K$, we can find another feasible solution $(\W',\Hb')$ such that $\mathcal{L}_{\lim}(\W,\Hb)-\mathcal{L}_{\lim}(\W',\Hb')\geq \epsilon^2$. To satisfies the minority constraints, one only needs to let $\frac{h_i}{\|h_i\|}=\frac{w_i}{\|w_i\|}$ and $\|h_i\|\gtrsim\frac{1}{\epsilon^{1+\delta_1}}$ for some $\delta_1>0$. In this case, if we take $\frac{n_B}{n_A}\le\epsilon^{1+\delta_1+\delta_2}$ for some $\delta_2>0$, then $\lim \frac{n_B}{n}\sum_{i=K/2+1}^K \|h_i\|^2=0$. Thus $\lim\mathcal{L}(\W,\Hb)-\mathcal{L}_{\lim}(\W',\Hb')\geq \epsilon^2$. At the same time, using the similar proof, we can proof that $\mathcal{L}_{\lim}(\W',\Hb')$ can become the limit objective function value for some limiting feasible solutions. Thus we knows that the limiting solution must satisfy the minority collapse condition, i.e.,$\|\w_k-\w_{k'}\|=0,\forall K/2<k<k'<K$.

\subsection{importance tempering on $h$}

In this subsection, we consider putting the temperature on the last layer feature $\Hb$. Following \cite{fang2021layer,ji2021unconstrained}, we consider the unconstrained layer-peeled model (ULPM), but here we cooperate the importance tempering on $\W$ and leads to the following new model
	\begin{equation}
	\begin{aligned}
	&\min_{\W,\Hb} \mathcal{L}(\W,\Hb) \\
	:= &\min_{\W,\Hb} -\sum_{k=1}^{K}\sum_{i=1}^{n_k} \log\left( \frac{\exp(\lambda_k\w_k^\top \h_{k,i})}{\sum_{j=1}^K\exp(\w_j^\top \lambda_k\h_{k,i})} \right).    
	\end{aligned}
    \end{equation}
    
 \cite{lyu2019gradient,ji2020gradient} proved that gradient descent on this loss will converge to the solution of the re-weighted minimum-norm separation problem

  \begin{equation}
\begin{aligned}
&\min _{\W, \Hb} \frac{1}{2}||\W||_F^2+\frac{1}{2}\sum_{i=1}^K n_K||h_k||^2\\
s.t.& \lambda_k\w_k^\top \h_{k}-\lambda_k\w_j^\top \h_{k} \geq 1,\quad  k\not=j\in[K].
\end{aligned}
\end{equation}
    
Similar to previous section, We first prove that the within-class variation of the activation
becomes negligible as these activation collapse to their class mean, \emph{i.e.} $h_{k,i_1}=h_{k,i_2}=\frac{1}{n_k}\sum_{i=1}^{n_k}h_{k,i}$ for all $k\in[K]$. If we have a feasible solution $(W,H)$ subject to $\|h_{k,i_1}-h_{k,i_2}\|\ge\epsilon>0$ for some $k\in[K],i_1,i_2\in [k]$. We can construct $\tilde{W}, \tilde{H}$ via letting $\tilde{h}_{k,i_1}=\tilde{h}_{k,i_2}=\frac{1}{2}(h_{k,i_1}+h_{k,i_2})$ and all the other vectors unchanged $\tilde{h}_{k^\prime,i^\prime}=h_{k^\prime,i^\prime}, \tilde{w}_{k^\prime}={w}_{k^\prime}$ for all $(k^\prime,i^\prime)\not = (k,i_1)$ or $(k,i_1)$. We first check that $\tilde{W}, \tilde{H}$ is also a feasible solution for
$$
w_k^\top \tilde{h}_{k,i_1}-w_j^\top \tilde{h}_{k,i_1}=w_k^\top \tilde{h}_{k,i_2}-w_j^\top \tilde{h}_{k,i_2}=\frac{1}{2}\left[(\w_k^\top \h_{k,i_1}-\w_j^\top \h_{k,i_2})+(\w_k^\top \h_{k,i_1}-\w_j^\top \h_{k,i_2})\right]\ge1.
$$
	\begin{theorem}
	If we applied the importance tempering $\lambda_k=C\sqrt{n_k},\forall k\in[K]$, where $C>0$ is a positive constant, then the optimal solution of the following constrained optimization problem satisfies neural collapse condition.
\begin{equation}
\begin{aligned}
\label{SVM}
&\min _{\W, \Hb} \frac{1}{2}||\W||_F^2+\frac{1}{2}||\Hb||_F^2\\
s.t.& \lambda_k\w_k^\top \h_{k,i}-\lambda_k\w_j^\top \h_{k,i} \geq 1,\quad  k\not=j\in[K],i\in [n_k].
\end{aligned}
\end{equation}
	\end{theorem}
\begin{proof}
	First we can find that the margin will not change if we minus a vector $a$ for all $w_j$, so if we denote the mean of classifier $\tilde{\w_i}=\w_i-\frac{1}{K}\sum_{i=1}^K\w_i$ and then we have $\w_{k}^\top \h_{k,i}-\w_{j}^\top \h_{k,i} = \tilde{\w}_{k}^\top \h_{k,i}-\tilde{\w}_{j}^\top \h_{k,i}$. Note that $\sum_{j=1}^K \tilde{\w}_j^\top \h_{k,i} = 0$ then sum this inequality over $j$ we have:
	\begin{equation*}
	(K-1) \tilde{\w}_{k}^\top \h_{k,i} - \sum_{j\not=k}\tilde{\w}_{j}^\top \h_{k,i}=K\tilde{\w}_{k}^\top \h_{k,i}\geq \frac{K-1}{\lambda_k}=\frac{(K-1)}{C\sqrt{n_k}}, \forall k\in[K],i\in [n_k].
	\end{equation*}
	By Cauchy inequality, we have:
	\begin{equation}
	\label{cauchy}
	\frac{1}{2}(\frac{1}{\sqrt{n_k}}||\tilde{\w}_k||^2+\sqrt{n_k}||\h_{k,i}||^2)\geq \tilde{\w}_{k}^\top \h_{k,i}\geq\frac{(K-1)}{KC\sqrt{n_k}}.
	\end{equation}
	Dividing $\sqrt{n_k}$ on both sides of (\ref{cauchy}) and sum over k and i we have:
	\begin{equation}
	\label{finequal}
	\frac{1}{2}(||\tilde{\W}||_F^2+||\Hb||_F^2)\geq \frac{(K-1)K}{KC},
	\end{equation}
	which gives us a lower bound for optimal value in optimization problem (\ref{SVM}). On the other hand, by the derivation of this lower bound we know that if the equality holds in equation (\ref{finequal}), then the mean of classifier equals to zero, i.e., $\sum_{i=1}^Kw_i=0,w_i=\tilde{w}_i$, and the equality in \ref{cauchy} holds for any $k\in[K]$ and $i\in[n_k]$, which implies that:
	\begin{equation}
	\label{eq: nc condition 1}
	    w_k=\sqrt{n_k} h_{k,i},\quad \|w_k\|^2=n_k\|h_{k,i}\|^2=\frac{K-1}{KC},\quad \forall k\in[K], i\in[n_k].
	\end{equation}
	Take above equation back into the constraint of the constrained optimization problem (\ref{SVM}), we can obtain that:
	\begin{equation}
	\label{eq: nc condition 2}
	    w_k^\top h_{k,i}=\frac{K-1}{K\lambda_k}, \quad w_j^\top h_{k,i}=-\frac{1}{K\lambda_k},\quad \cos(w_k,w_j)=-\frac{1}{K-1}.
	\end{equation}
	Combine equation (\ref{eq: nc condition 1}) and (\ref{eq: nc condition 2}) together we can obtain that the optimal solution satisfies neural collapse conditions. 
\end{proof}

\subsection{importance tempering on $\W$}
\label{subsection:itw}

In this subsection, we consider putting the temperature on the last layer classifier $\W$. Following \cite{fang2021layer,ji2021unconstrained}, we consider the unconstrained layer-peeled model (ULPM), but here we cooperate the importance tempering on $\W$ and leads to the following new model
	\begin{equation}
	\begin{aligned}
	&\min_{\W,\Hb} \mathcal{L}(\W,\Hb) \\
	:= &\min_{\W,\Hb} -\sum_{k=1}^{K}\sum_{i=1}^{n_k} \log\left( \frac{\exp(\lambda_k\w_k^\top \h_{k,i})}{\sum_{j=1}^K\exp(\lambda_j\w_j^\top \h_{k,i})} \right).    
	\end{aligned}
    \end{equation}
    
 \cite{lyu2019gradient,ji2020gradient} proved that gradient descent on this loss will converge to the solution of the re-weighted minimum-norm separation problem

  \begin{equation}
\begin{aligned}
&\min _{\W, \Hb} \frac{1}{2}||\W||_F^2+\frac{1}{2}\sum_{i=1}^K n_K||h_k||^2\\
s.t.& \lambda_k\w_k^\top \h_{k}-\lambda_j\w_j^\top \h_{k} \geq 1,\quad  k\not=j\in[K].
\end{aligned}
\end{equation}

We want to proof that the classifier will form a ETF with largest possible angles. However, the solution of the non-convex problem does not lies in a compact set and leads to technical problems. In this section, we will discuss the intuition of why we think the limiting classifier will become a ETF. Similar to previous section, We first prove that the within-class variation of the activation
becomes negligible as these activation collapse to their class mean, \emph{i.e.} $h_{k,i_1}=h_{k,i_2}=\frac{1}{n_k}\sum_{i=1}^{n_k}h_{k,i}$ for all $k\in[K]$. If we have a feasible solution $(W,H)$ subject to $\|h_{k,i_1}-h_{k,i_2}\|\ge\epsilon>0$ for some $k\in[K],i_1,i_2\in [k]$. We can construct $\tilde{W}, \tilde{H}$ via letting $\tilde{h}_{k,i_1}=\tilde{h}_{k,i_2}=\frac{1}{2}(h_{k,i_1}+h_{k,i_2})$ and all the other vectors unchanged $\tilde{h}_{k^\prime,i^\prime}=h_{k^\prime,i^\prime}, \tilde{w}_{k^\prime}={w}_{k^\prime}$ for all $(k^\prime,i^\prime)\not = (k,i_1)$ or $(k,i_1)$. We first check that $\tilde{W}, \tilde{H}$ is also a feasible solution for
$$
w_k^\top \tilde{h}_{k,i_1}-w_j^\top \tilde{h}_{k,i_1}=w_k^\top \tilde{h}_{k,i_2}-w_j^\top \tilde{h}_{k,i_2}=\frac{1}{2}\left[(\w_k^\top \h_{k,i_1}-\w_j^\top \h_{k,i_2})+(\w_k^\top \h_{k,i_1}-\w_j^\top \h_{k,i_2})\right]\ge1.
$$

At the same time, the objective function will decay at least $\epsilon^2$ for we have $\frac{1}{2}[\|\tilde{h}_{k,i_1}\|^2+\|\tilde{h}_{k,i_2}\|^2]=\frac{1}{2}[\|{h}_{k,i_1}\|^2+\|{h}_{k,i_2}\|^2-2\|{h}_{k,i_1}-{h}_{k,i_2}\|^2]\le\frac{1}{2}[\|{h}_{k,i_1}\|^2+\|{h}_{k,i_2}\|^2]-\epsilon^2$. Thus we know that the within-class variation of the activation
becomes negligible as these activation collapse to their class mean, \emph{i.e.} $h_{k,i_1}=h_{k,i_2}=\frac{1}{n_k}\sum_{i=1}^{n_k}h_{k,i}$ for all $k\in[K]$. Thus problem (\ref{SVM:appendix}) can be formulated as
 \begin{equation}
\begin{aligned}
&\min _{\W, \Hb} \frac{1}{2}||\W||_F^2+\frac{1}{2}\sum_{i=1}^K n_K||h_k||^2\\
s.t.& \lambda_k\w_k^\top \h_{k}-\lambda_j\w_j^\top \h_{k} \geq 1,\quad  k\not=j\in[K].
\end{aligned}
\end{equation}
To balance and simplify the problem, we further consider another reparameterization of $\W$. We substitute $w_k$ by $\sqrt{Knn_An_B/2}w_k$, where $n=\sum_{i=1}^K n_i$ is the total number of data, and lead the following problem
\begin{equation}
\begin{aligned}
\label{SVM:reweightedonw}
&\min _{\W, \Hb} \sum_{i=1}^{K/2}\frac{n_B}{n}||w_i||^2+\sum_{i=K/2+1}^{K}\frac{n_A}{n}||w_i||^2+\sum_{i=1}^K \frac{2n_k}{n}||h_k||^2\\
s.t.& \w_k^\top \h_{k}-\w_j^\top \h_{k} \geq 1,\quad  k\not=j\in[K].
\end{aligned}
\end{equation}

Note that this reparameterization will not change any conclusion of the directional convergence. We first discuss the intuitive interpretation of this optimization problem.  Under the limit $\frac{n_A}{n_B}\rightarrow\infty$, problem \ref{SVM:reweightedonw} can be considered as minimizing the norm of classifier corresponding to the minority classes and the norm of the features corresponding to the majority classes.

We first prove that if the global solution of problem (\ref{SVM:reweightedonw}) have a directional limit when $\frac{n_A}{n_B}\rightarrow\infty$, then we will have $\lim_{R\rightarrow\infty} \|w_i\|=\infty, \lim_{R\rightarrow\infty} \|h_i\|=0 (1\le i\le K/2)$ for all majority classes and $\lim_{R\rightarrow\infty} \|w_i\|=0, \lim_{R\rightarrow\infty} \|h_i\|=\infty (K/2+1\le i\le K)$ for all minority classes. First we prove that

$$
\lim_{R\rightarrow\infty}\sum_{i=1}^{K/2}\frac{n_B}{n}||w_i||^2+\sum_{i=K/2+1}^{K}\frac{n_A}{n}||w_i||^2+\sum_{i=1}^K \frac{2n_k}{n}||h_k||^2\rightarrow 0.
$$

This is because once $w_j^\top h_k\le 0$ for all pairs of $k,j$ (this is feasible for the ETF is a simple example), we can always keep  $\frac{w_k}{\|w_k\|}=\frac{h_k}{\|h_k\|}$ and scale $\|w_j\|$ to zero, $\|w_j\|\|h_k\|$ to infinity and $\frac{n_B}{n_A}\|h_k\|$ to zero. In this case $w_kh_k\ge 0$ and $-w_j^\top h_k\ge 1$ for all pairs of $k,j$. Thus we can keep this sequence always satisfies the constraints and limits $\sum_{i=1}^{K/2}\frac{n_B}{n}||w_i||^2+\sum_{i=K/2+1}^{K}\frac{n_A}{n}||w_i||^2+\sum_{i=1}^K \frac{2n_k}{n}||h_k||^2$ to zero. For $\lim_{R\rightarrow\infty} \sum_{i=1}^{K/2}\frac{n_B}{n}||w_i||^2+\sum_{i=K/2+1}^{K}\frac{n_A}{n}||w_i||^2+\sum_{i=1}^K \frac{2n_k}{n}||h_k||^2\rightarrow 0$, we knows that $ \lim \|h_i\|=0 (1\le i\le K/2)$ and $\lim \|w_i\|=0 (K/2+1\le i\le K)$. To satisfies the constraints, we know that $ \lim \|w_i\|=\infty (1\le i\le K/2)$ and $\lim \|h_i\|=\infty (K/2+1\le i\le K)$.

Then we will prove that $\lim_{R\rightarrow\infty} w_k^\top h_k\ge 1 (1\le i\le K/2)$ for all majority classes. This is because for $w_j (K/2+1\le j\le K) $ for majority class and $h_k (1\le i\le K/2)$ for the majority classes, we have $0\le |w_jh_k|\le \|w_j\|\|h_k\| \rightarrow 0$. Thus we have
$$
\lim_{R\rightarrow\infty} w_k^\top h_k\ge 1- \lim_{R\rightarrow\infty} |w_k^\top h_k|=1(1\le i\le K/2).
$$
Thus we have 
\begin{equation}
    \begin{aligned}
    1 = \lim_{R\rightarrow\infty} w_k^\top h_k \le \lim_{R\rightarrow\infty} \frac{n}{\sqrt{n_An_B}} (\frac{n_B}{n}\|\tilde{w_i}\|^2+\frac{n_A}{n}\|\tilde{w_i}\|^2).
    \end{aligned}
    \label{eq:lowermaj}
\end{equation}

For minority classes, we first decompose $\tilde{w_i}= w_i - \frac{2}{K}\sum_{i=K/2+1}^{K}w_i$, then $\tilde{w_i}^\top h_{i}-\tilde{w_j}^\top h_{i}=w_i^\top h_i-w_j^\top h_i\ge 1$ and $\sum_{i=K/2}^{K}\|w_i\|^2= \sum_{i=K/2}^{K}\|\tilde{w_i}\|^2+\frac{K}{2}\|\frac{2}{K}\sum_{i=K/2+1}^{K}w_i\|^2$. At the same time, we knows that 
\begin{equation}
    \begin{aligned}
    \frac{K}{2}-1\le\sum_{j=K/2+1,j\not=i}^{K} \tilde{w_i}^\top h_{i}-\tilde{w_j}^\top h_{i} = (\frac{K}{2}-1)  \tilde{w_i}^\top h_{i} -(\sum_{j=K/2+1,j\not=i}^{K}\tilde{w_j})^\top h_i = \frac{K}{2} \tilde{w_i}^\top h_{i} \le \frac{nK}{\sqrt{n_An_B}} (\frac{n_B}{n}\|\tilde{w_i}\|^2+\frac{n_A}{n}\|h_i\|^2)
    \end{aligned}
    \label{eq:lowermin}
\end{equation}

At the same time, the equality of (\ref{eq:lowermaj}) and (\ref{eq:lowermaj}) can be achieved when $\frac{w_i}{\|w_i\|}=\frac{h_i}{\|h_i\|}$ and 
$$\lim_{R\rightarrow \infty}\cos(\bar{{\h}}_k,\bar{{\h}}_j)=-\frac{1}{\frac{K}{2}-1},\quad||\bar{{\h}}_k||=||\bar{{\h}}_j||,$$ 
    for all $\forall K/2+1\le k\not = j\le K$.
    
Finally, we only need to prove that the solution also satisfies the other constraints. This is because if we can let $w_j^\top h_i\le 0 (\forall 1\le j\le \frac{K}{2}, \frac{K}{2}+1\le j\le K)$ (the constraint that we can classify the minority data from the majority data) then $w_j^\top h_i\rightarrow \infty$. This can be easily satisfied, for we can use first half of the feature to construct the ETF for minority classes and the second part to construct the majority classes. Once this happens, we have $w_i^\top h_i-w_i^\top h_j\rightarrow \infty \ge 1$.

\section{Proof for Synthetic Dataset from \cite{sagawa2020investigation}}\label{appendix:sagawanorm}

In this section, we present the proof of Theorem \ref{theorem:sagawaexample}. We use the synthetic dataset \cite{sagawa2020investigation}. 
\begin{itemize}
\setlength{\itemsep}{0pt}
\setlength{\parsep}{0pt}
\setlength{\parskip}{0pt}
    \item $x_c|y\sim \mathcal{N}(\mu_cy,(\mu_c\sigma_c)^2),x_s|a\sim \mathcal{N}(\mu_sa,(\mu_s\sigma_c)^2),$
    \item $x_n\sim \mathcal{N}\left(0,\frac{\sigma_n^2n}{N}I_N\right)$
\end{itemize}
where $\sigma_c,\sigma_s,\sigma_n,\mu_1,\mu_2$ are five constants. $\mu_1,\mu_2$ denotes the scale of the features. When the features are larger, the classifier will need a smaller norm to achieve a margin of a fixed size. Due to the inductive bias of training overparameterized models, this task is easier to learn. $\sigma_c,\sigma_s$ denote the noise in the features. Smaller noise means the feature contains more information, \emph{i.e.} a smaller fraction of the data will need to be memorized when using this feature. Different from \cite{sagawa2020investigation}, we add a normalizing factor $n$ in the noisy feature, \emph{i.e.} the $\sigma_n$ in \cite{sagawa2020investigation} is $\frac{\sigma_n}{\sqrt{n}}$ in our paper. We introduce this normalization so that the cost to memorize all the data is $O(1)$ but not $O(n)$ in \cite{sagawa2020investigation}'s setting. In this regime, we could consider how the norm of core classifier and norm of spurious classifier affects the problem. If one considers the limit $\sigma_n\rightarrow 0$, \emph{i.e.} the regime that inductive bias emphasize more on the cost to memorize the data, the result will go back to \cite{sagawa2020investigation}'s result. In Appendix \ref{appendix:betterthanrandom}, we go back to \cite{sagawa2020investigation}'s example and proof that importance tempering can achieve better than random results while \cite{sagawa2020investigation} proves that overparametrized model will have error larger than $\frac{2}{3}$. We first provide several concentration inequalities for the following proves
\begin{lemma}[\cite{sagawa2020investigation} Lemma 8, Lemma 9.] \label{thm: sagawa lemmas} For $N=\Omega(\text{poly}(n))$, with probability greater than $1-1/2000$,
$$
|x_n^{(i)}\cdot x_n^{(j)}|\le\frac{\sigma_n^2}{n^6}, \quad  \left(1-O(\frac{1}{n^3})\right)\sigma_n^2\le \|x_n^{(i)}\|\le \left(1+O(\frac{1}{n^3})\right)\sigma_n^2
$$
for all $1\le i\not=j\le n$
\end{lemma}
Following \cite{sagawa2020investigation}, for any estimator $\hat w=[\hat w_c,\hat w_s,\hat w_n], \hat w_c,\hat w_s\in\mathbb{R}$ and $\hat w_n\in\mathbb{R}^N$, we decompose $\hat w_n$ using representer theorem,
$$
\hat w_n=\sum_{i=1}^n\frac{\alpha^{(i)}}{\sigma_n}x_n^{(i)}.
$$
For we can separate all the data via setting $\alpha^{(i)}=1$ and $\hat w_c=\hat w_s=0$. Thus we can consider all estimator with $O(n)$ norm, this leads to $\alpha^{(i)}\le O(n)$. Thus for all $x^{(i)}$, we have
$$
\hat w_n\cdot x_n^{(i)} = \frac{\alpha^{(j)}}{\sigma_n^2}\|x_n^{(i)}\|^2 + \sum_{j=1,j\not= i}^n \frac{\alpha^{(j)}}{\sigma_n^2}{x_n^{(j)}}^\top x_n^{(i)} = \alpha^{(j)}+O(\frac{1}{n^2}).
$$

\subsection{Compare $\|w^c\|$ and $\|w^s\|$}
\blockcomment{
Let $\lambda > 1$ be the margin for the minority class. The parentheses refer to the corresponding lemmas/propositions of Sagawa's paper.

\yplu{here is an upper bound of only using core feature}
(Lemma 1) Bound $\|w\|$ with the norm of a specific separator; this bound will depend on $\lambda$. Let’s say $\|w\|^2 \leq f(\lambda)$, and note that we should probably have $f(\lambda) = \Omega(\lambda^2 n_{\textrm{min}})$. \yplu{no ,for the upper bound, we need to memorize the majority data, the scale I guess is $\Omega(n_{\text{maj}})$}

\begin{itemize}
    \item $w_c=u_1$
    \item $w_s=0$
    \item $\alpha^{(i)}(w)=y^{i}$ for $i\in G_{\text{maj}}$
    \item $\alpha^{(i)}(w)=0$ for $i\in G_{\text{min}}$
\end{itemize}

\yplu{here is an lower bound of only using spurious feature}
(Lemma 2) Lower bound $\|w\|$ with respect to $\delta_{\min}$: $\|\hat{w}\|^2 \gtrsim \frac{\gamma^2}{\sigma_n^2} \delta_{\min}(\hat{w}, \gamma) n_{\min}$

(Lemma 4) Combine 1 and 2 to get an upper bound: $\delta_{\min} \lesssim \frac{f(\lambda)\sigma_n^2}{\gamma^2 n_{\min}}$

(Proposition 3) Get a lower bound: $\delta_{\min}(\hat{w}, \gamma) \gtrsim \Phi\left(\frac{\lambda - \gamma - \hat{w}_c + \hat{w}_s}{\sqrt{\hat{w}_c^2 \sigma_c^2 + \hat{w}_s^2 \sigma_s^2}}\right)$.

Combining these, we get 
\begin{equation} \label{eq: sagawa 1}
    \Phi\left(\frac{\lambda - \gamma - \hat{w}_c + \hat{w}_s}{\sqrt{\hat{w}_c^2 \sigma_c^2 + \hat{w}_s^2 \sigma_s^2}}\right) \lesssim \frac{f(\lambda)\sigma_n^2}{\gamma^2 n_{\min}}.
\end{equation}
Let’s look at some asymptotics. If we rearrange and use the approximation $\Phi^{-1}(1-y) \approx \sqrt{2 \ln \frac1y}$ \href{https://math.stackexchange.com/questions/2964944/asymptotics-of-inverse-of-normal-cdf](https://math.stackexchange.com/questions/2964944/asymptotics-of-inverse-of-normal-cdf}{(source)}, we get

\begin{equation} \label{eq: sagawa 2}
    \frac{\hat{w}_c - \hat{w}_s}{\sqrt{\hat{w}_c^2 \sigma_c^2 + \hat{w}_s^2 \sigma_s^2}} \geq \frac{\lambda-\gamma}{\sqrt{\hat{w}_c^2 \sigma_c^2 + \hat{w}_s^2 \sigma_s^2}} + \sqrt{2 \ln \frac{\gamma^2 n_{\min}}{f(\lambda)\sigma_n^2}}
\end{equation}

\yplu{I guess we use this equation to bound the lower bound of the solution of only using spurious feature, so here we can assume the $w_c=0$ and see my se-con comment}

\zi{These asymptotics also might not be accurate since this holds when $y \rightarrow 0^+$. If $\gamma^2 n_{\min}/f(\lambda) \sigma_n^2 \approx 1$ then maybe the result can still be ok.} \yplu{The $y$ is selected by ourself, in shiori, we did not consider the limit?}

This does not look good. Note that we need the RHS of \eqref{eq: sagawa 1} to be $<1$ in order for the bound to be nontrivial. Combined with the (likely) lower bound on $f(\lambda)$, this should suggest that $\gamma = c\lambda$ for some $c > 1$. In Sagawa et al.'s regime, $\sigma_c$ can be a constant, so unless $\hat{w}_c$ becomes huge, the RHS of \eqref{eq: sagawa 2} will be negative and we do not get a useful bound.

\yplu{here we are wanting $\hat w_c=0$, so that the final bound may relies on$\Phi(-\frac{1}{\sigma_{s}})$ like what sagawa's (86)? and finally leads to a lower bound of the norm at the scale $\gamma \Phi(-\frac{1}{\sigma_{s}})$}

For Error
$$
\text{Err}_{\text{wg}}(\hat w)\approx \Phi\left(\frac{|w_{\text{spu}}|-w_{\text{core}}}{\sqrt{\hat{w}_c^2 \sigma_c^2 + \hat{w}_s^2 \sigma_s^2}}\right)
$$

\yplu{I guess we should also have a upper bound for $\gamma$? Otherwise we can't learn the majority group?}

$\delta_{\min}(\hat{w}, \gamma) \approx \Phi\left(\frac{\lambda - \gamma - \hat{w}_c + \hat{w}_s}{\sqrt{\hat{w}_c^2 \sigma_c^2 + \hat{w}_s^2 \sigma_s^2}}\right)$ \yplu{bad notation $\gamma$ here, at the same time, the gamma here may be important parameter to select}

$\delta_{\text{maj}}(\hat{w}, \gamma) \approx \Phi\left(\frac{1 - \gamma^2 - \hat{w}_c - \hat{w}_s}{\sqrt{\hat{w}_c^2 \sigma_c^2 + \hat{w}_s^2 \sigma_s^2}}\right)$

\begin{lemma}
$\delta_{\min}(\hat{w}, \gamma) \approx \Phi\left(\frac{\lambda - \gamma^2 - \hat{w}_c + \hat{w}_s}{\sqrt{\hat{w}_c^2 \sigma_c^2 + \hat{w}_s^2 \sigma_s^2}}\right)$,$\delta_{\text{maj}}(\hat{w}, \gamma) \approx \Phi\left(\frac{1 - \gamma^2 - \hat{w}_c - \hat{w}_s}{\sqrt{\hat{w}_c^2 \sigma_c^2 + \hat{w}_s^2 \sigma_s^2}}\right)$
\end{lemma}

\begin{proof}
Intuitively we want to show that the fraction of majority data memorized depends on $\hat w_s-\hat w_c$ and the fraction of minority data memorized depends on $\hat w_s+\hat w_c$.

\paragraph{Bounding $\delta_{\text{maj}}$} Since the training point $i$ is separated, we have
$$
\hat w_{c}(1+\sigma_c z_1)+\hat w_s(1+\sigma_s z_2)+\left(\sum_j \alpha^{(j)}(\hat w)x_n^{(j)}\right)^\top x_n^{(i)}\ge 1
$$

This implies 

$$
\hat w_{c}(1+\sigma_c z_1)+\hat w_s(1+\sigma_s z_2)\ge 1-(1+c)\sigma_n^2\alpha^{(i)}(\hat w)-c
$$

Thus
$$
\mathbb{P}(\alpha^{(i)}(\hat w)\le\frac{\gamma^2}{\sigma_n^2})\approx \Phi\left(\frac{1 - \gamma^2 - \hat{w}_c - \hat{w}_s}{\sqrt{\hat{w}_c^2 \sigma_c^2 + \hat{w}_s^2 \sigma_s^2}}\right)
$$

\paragraph{Bounding $\delta_{\text{min}}$} Since the training point $i$ is separated, we have
$$
\hat w_{c}(1+\sigma_c z_1)+\hat w_s(-1+\sigma_s z_2)+\left(\sum_j \alpha^{(j)}(\hat w)x_n^{(j)}\right)^\top x_n^{(i)}\ge 1
$$

This implies 

$$
\hat w_{c}(1+\sigma_c z_1)+\hat w_s(-1+\sigma_s z_2)\ge \lambda-(1+c)\sigma_n^2\alpha^{(i)}(\hat w)-c
$$

Thus
$$
\mathbb{P}(\alpha^{(i)}(\hat w)\le\frac{\gamma^2}{\sigma_n^2})\approx \Phi\left(\frac{\lambda - \gamma^2 - \hat{w}_c + \hat{w}_s}{\sqrt{\hat{w}_c^2 \sigma_c^2 + \hat{w}_s^2 \sigma_s^2}}\right)
$$

\end{proof}

The previous proof consider $\Phi^{-1}(\delta_{\text{maj}})+\Phi^{-1}(\text{Err}_{\text{wg}})\approx \frac{-2\hat w_c}{\sqrt{\hat{w}_c^2 \sigma_c^2 + \hat{w}_s^2 \sigma_s^2}}$

Idea: the first thing is we need to classify right most of the minority data to lower bound the worst group accuracy, \emph{i.e.} upper bound  $\Phi\left(\frac{\lambda - \gamma - \hat{w}_c + \hat{w}_s}{\sqrt{\hat{w}_c^2 \sigma_c^2 + \hat{w}_s^2 \sigma_s^2}}\right)$ to lower bound $\Phi\left(\frac{w_{\text{spu}}-w_{\text{core}}}{\sqrt{\hat{w}_c^2 \sigma_c^2 + \hat{w}_s^2 \sigma_s^2}}\right)$ \yplu{I think this looks true, I don't memory my minority data, so that I can classify them right.} This can be done via constructing an estimator
\begin{itemize}
    \item $w_c=u_1$
    \item $w_s=0$
    \item $\alpha^{(i)}(w)=y^{i}$ for $i\in G_{\text{maj}}$
    \item $\alpha^{(i)}(w)=0$ for $i\in G_{\text{min}}$
\end{itemize}

\zi{The problem is that when we use the core feature, we will still need to memorize at least a constant fraction of both the minority and majority points. This is because the variance of the core feature is $\Omega(1)$, so at least a constant fraction of the points have the wrong sign for the core feature. This is the main reason that I'm having trouble finding a good upper bound for the estimator.} \yplu{Yes, there must be a constant error, I think it would be great to control it at 0.4 or smaller than the lower bound 0.67 for the upper bound}
We also need to upper bound the majority group data, we can do similar thing, now we need upper bound of the $\gamma$

maybe now we need to consider
\begin{itemize}
    \item $w_c=0$
    \item $w_s=u_2$
    \item $\alpha^{(i)}(w)=0$ for $i\in G_{\text{maj}}$
    \item $\alpha^{(i)}(w)=\gamma y^{i}$ for $i\in G_{\text{min}}$
\end{itemize}

using the two estimator we can have the relationship between the selection of $\gamma$ and the ratio $\frac{u_1}{u_2}$ (\yplu{$\frac{u_1}{u_2}$ can be controlled via control the number of the features})

\yplu{We can ignore the Gaussian framework, but we need to see the best $\gamma$ is rely on the ratio $\frac{\|u_1\|}{\|u_2\|}$}
}

Let $\lambda > 1$ be the margin for the minority class enforced by influence temperature. Furthermore, write $x^{(i)}_c = y^{(i)} + z^{(i)}_c$ where $z^{(i)}_c \sim N(0, \sigma)$. To simplify our proof, we use $w_s,w_c$ to denote $\frac{w_s}{\mu_s},\frac{w_c}{\mu_c}$ and the norm of $w$ will be defined as $\frac{w_s^2}{\mu_s^2}+\frac{w_s^2}{\mu_s^2}+\|w_n\|^2$. Suppose that $w_c$ and $w_s$ are fixed. By the near-orthonormality of the $x^{(i)}_n$ and the margin constraint, we can actually determine $\alpha^{(i)}$ almost exactly:
$$ \alpha^{(i)} = \begin{cases} y^{(i)}(1 - w_s - w_c - w_c z^{(i)}_c)_++O(\frac{1}{n^2}) & i \in G_{\textrm{maj}} \\ y^{(i)}(\lambda + w_s - w_c - w_c z^{(i)}_c)_++O(\frac{1}{n^2}) & i \in G_{\textrm{min}} \end{cases} $$
(For a complete proof of this fact, see Lemma~\ref{thm: a^i error} below.) This allows us to compute the expected norm of a separator $w$ in terms of $w_s$ and $w_c$:
\begin{equation} \label{eq: separator norm}
\mathbb{E}[\|w\|^2] = \frac{w_s^2}{\mu_s^2}+\frac{w_c^2}{\mu_c^2} + \frac{n_{\textrm{maj}}}{n\sigma_n^2} \mathbb{E}[(1 - w_s - w_c + w_c z)_+^2] + \frac{n_{\textrm{min}}}{n\sigma_n^2} \mathbb{E}[(\lambda + w_s - w_c + w_c z)_+^2]+O(\frac{1}{n^2}). 
\end{equation}
We first consider the case when $w_c = 0$ and we may only use the spurious feature. In this case, there is no randomness in $\|w\|^2$ (all of the randomness comes from the core feature) and we can compute the expectation exactly:
$$ \mathbb{E}[\|w\|^2] =\frac{w_s^2}{\mu_s^2}+ \frac{p_{\textrm{maj}}}{\sigma_n^2}(1 - w_s)^2 + \frac{p_{\textrm{min}}}{\sigma_n^2}(\lambda + w_s)^2 $$
provided that $w_s \in [-\lambda, 1]$. This is a quadratic with minimum at $w_s = \frac{(\frac{p_{\textrm{maj}}}{\sigma_n^2} - \lambda \frac{p_{\textrm{min}}}{\sigma_n^2})}{\frac{1}{\sigma_n^2} + \frac{1}{\mu_s^2}}$ (note that this falls within the required range), which yields
\begin{align}
\mathbb{E}[\|w^{\textrm{use-spu}}\|^2] &\ge \frac{\pmaj}{\sigma_n^2} + \frac{\lambda^2 p_{\min}}{\sigma_n^2} - \frac{(\frac{\pmaj}{\sigma_n^2} - \lambda \frac{p_{\min}}{\sigma_n^2})^2}{\frac{1}{\sigma_n^2} + \frac{1}{\mu_s^2}} 
 \label{eq: use-spu 1}
\end{align}
\blockcomment{
Let $\lambda = \sqrt{\nmaj/\nmin}$ as suggested by Theorem~\ref{thm: it gen bound}. Recall that $\pmaj = \nmaj / n$, from which it follows that $\nmin = \frac{1-\pmaj}{\pmaj}\nmaj$. Using this fact, some simple algebra shows that \eqref{eq: use-spu 1} implies
\begin{equation} \label{eq: use-spu 2}
\mathbb{E}[\|\spu\|^2] \geq \left(c^2 + 2\pmaj \sqrt{\frac{1-\pmaj}{\pmaj}}\right)\nmaj.
\end{equation}
(This inequality is tight up to an $n/(n+1)$ factor on one of the terms.) 
}

Next, we turn our attention to $\core$. The terms in \eqref{eq: separator norm} all take the form $\mathbb{E}[(a + bz)_+^2]$, where $a$ and $b$ are constants and $z\sim N(0,\sigma^2)$. This is a Gaussian integral, and some elementary manipulations show that, for $b > 0$,
\begin{equation} \label{eq: gauss integral}
\mathbb{E}[(a + bz)_+^2] = (a^2 + b^2\sigma^2) \Phi\left(\frac{a}{b\sigma}\right) + \frac{ab\sigma}{\sqrt{2\pi}}e^{-\frac{a^2}{2b^2\sigma^2}}.
\end{equation}
Note that when $a = 0$, this equation simplifies to $\frac12 b^2\sigma^2$. Thus, taking $w_s = 0$ and $w_c = 1$, we obtain
\begin{align}
\mathbb{E}[\|\core\|^2] &\leq \frac{1}{\mu_c^2} + \frac12 \frac{\nmaj}{n\sigma_n^2} + \left[((\lambda - 1)^2 + \sigma^2)\Phi(\frac{\lambda - 1}{\sigma}) + \frac{\lambda - 1}{\sqrt{2\pi}} e^{-\frac{(\frac{\lambda - 1}{\sigma})^2}{2}}\right] \frac{\nmin}{n\sigma_n^2} \nn \\[5pt]
&\leq \frac{1}{\mu_c^2} + \frac12 \frac{\nmaj}{n\sigma_n^2} + (\lambda^2 - 2\lambda + 1+\sigma^2+ \frac{\lambda - 1}{\sqrt{2\pi}})\frac{\nmin}{n\sigma_n^2} \\[5pt]
&= \frac{1}{\mu_c^2} + \left(\frac12 \frac{\pmaj}{\sigma_n^2} + \left(\lambda^2-2\lambda+1+\sigma^2+ \frac{\lambda - 1}{\sqrt{2\pi}}\right)(\frac{1-\pmaj}{\sigma_n^2}) \right). \label{eq: use-core 1}
\end{align}

Combining \eqref{eq: use-spu 1} and \eqref{eq: use-core 1},  we see that the max margin solution will prefer $\core$ over $\spu$ provided that
\begin{equation*}
\frac{\pmaj}{\sigma_n^2} + \frac{\lambda^2 p_{\min}}{\sigma_n^2} - \frac{(\frac{\pmaj}{\sigma_n^2} - \lambda \frac{p_{\min}}{\sigma_n^2})^2}{\frac{1}{\sigma_n^2} + \frac{1}{\mu_s^2}} \ge \frac{1}{\mu_c^2} + \left(\frac12 \frac{\pmaj}{\sigma_n^2} + \left(\lambda^2-2\lambda+1+\sigma^2+ \frac{\lambda - 1}{\sqrt{2\pi}}\right)(\frac{1-\pmaj}{\sigma_n^2}) \right).
\end{equation*}
This gives us a quadratic in $\lambda$ inequality
\begin{equation}
    \begin{aligned}
    &\left(\frac{\frac{p_{\min}^2}{\sigma_n^4}}{\frac{1}{\sigma_n^2} + \frac{1}{\color{purple}\mu_s^2}}\right)\lambda^2\\
    &-2\left((1-\frac{1}{2\sqrt{2\pi}})\frac{p_{\min}}{\sigma_n^2}+\frac{\frac{p_{\min}\pmaj}{\sigma_n^4}}{\frac{1}{\sigma_n^2} + \frac{1}{\color{purple}\mu_s^2}}\right)\lambda\\
    &+{\color{orange}\frac{1}{\mu_c^2}}+\frac{1}{2}\frac{\pmaj}{\sigma_n^2}+(1-\frac{1}{\sqrt{2\pi}}+{\color{red}\sigma^2})\left(\frac{1-\pmaj}{\sigma_n^2}\right)-\frac{\pmaj}{\sigma_n^2}-\frac{\frac{\pmaj^2}{\sigma_n^4}}{\frac{1}{\sigma_n^2} + \frac{1}{\mu_s^2}} \le 0
    \end{aligned}
    \label{finallambda}
\end{equation}
First of all, we have $\left(\frac{\frac{p_{\min}^2}{\sigma_n^4}}{\frac{1}{\sigma_n^2} + \frac{1}{\mu_s^2}}\right)\ge 0$. If $\sigma_n$ is small enough, \emph{i.e.} the inductive bias emphasis more on reducing the effort to memorize data and the terms at the scale $\frac{1}{\sigma_n^2}$ dominates, then $\lambda$ satisfies (\ref{sagawalambda}) will always satisfies (\ref{finallambda}).
\begin{equation}
\frac{(1-\frac{1}{2\sqrt{2\pi}})p_{\min}+p_{\min}\pmaj - \sqrt{\Delta}}{p_{\min}^2} \leq \lambda \leq \frac{(1-\frac{1}{2\sqrt{2\pi}})p_{\min}+p_{\min}\pmaj + \sqrt{\Delta}}{p_{\min}^2} .
\label{sagawalambda}
\end{equation}
then $\|w_s\|\le\|w_c\|$, where $\Delta=((1-\frac{1}{2\sqrt{2\pi}})p_{\min}+p_{\min}\pmaj)^2-( p_{\min}^2)((1-\frac{1}{\sqrt{2\pi}}+\sigma)p_{\min}-\pmaj^2-\frac{1}{2}\pmaj)$. Note that if $\sigma$ small enough and $\pmaj$ large enough, then $\Delta>0$. This indicates that there exists a temperature $\lambda$ prefer to use the core feature in this regime. 

\begin{remark}\label{remark:selection} In this remark, we will discuss how different parameters changes the selection of temperature $\lambda$, \emph{i.e.} the solution of (\ref{finallambda}). If using the spurious feature  to classifier is harder (${\color{purple}{\mu_s}}$ becomes smaller), the quadratic coefficient becomes larger and the abstract value of linear coefficient becomes smaller. This indicates that the mean of the solution $\lambda$ will become smaller via Vieta's formulas. If the information of core feature decrease (${\color{red}\sigma}$ becomes larger) or using the core feature to classifier is harder (${\color{orange}{\mu_c}}$ becomes smaller), (\ref{finallambda}) will becomes harder to satisfies. If the core feature have too less information (${\color{red}\sigma}$ is too larger) or the core task is hard enough (${\color{orange}{\mu_c}}$ is large enough), even importance tempering cannot fix the bias.  At the same time, the smallest possible $\lambda$ will becomes larger. This indicates that lager temperature is needed. 
\end{remark}

\subsection{Accuracy of importance tempering on \cite{sagawa2020investigation}'s Example.}
\label{appendix:betterthanrandom}

In this section, we present the proof of Theorem \ref{theorem:betterthanrandom}, which indicates
that importance tempering can achieve better random accuracy on \cite{sagawa2020investigation}. At
the same time, \cite{sagawa2019distributionally} proved that ERM/importance weighting will have
error larger than $\frac{2}{3}$.

In all of the proofs that follow, we use big-O notation to analyze the behavior of various quantities as $n$ gets large. Thus quantities such as $1/(1-p)$, $\lambda$, etc. will be hidden by $O(1)$ as we assume that they do not grow with the sample size $n$.

Our first goal will be to show that $\|w\|^2/n$ concentrates around its expectation for any minimum-norm separator $w$. This in turn will allow us to just analyze the expected norm to prove that importance temperature achieves better than random worst-group accuracy. In all of the lemmas that follow, the result holds for sufficiently large $n$; we omit this from the lemma statements for brevity.

\begin{lemma}
If $w$ is a minimum norm separator and the high-probability results of Lemma~\ref{thm: sagawa lemmas} hold, then $\|w\|^{2}=O(n)$.
\end{lemma} 

\begin{proof}
Define $\alpha^{(i)}=2y^{(i)}$ if $i\in\mathcal{{I}_{\textrm{{maj}}}}$
and $\alpha^{(i)}=2\lambda y^{(i)}$ if $i\in\mathcal{{I}_{\textrm{{min}}}}$.
Observe that $\|w\|^{2}=O(n)$ and $y^{(i)}w\cdot x^{(i)}$ satisfies
all of the margin requirements for large enough $n$ when the high probability
events of Lemma~\ref{thm: sagawa lemmas} hold. This completes
the proof. 
\end{proof}

\begin{lemma} \label{thm: |a^i| bound}
If the high-probability results of Lemma~\ref{thm: sagawa lemmas} hold, then a minimum norm separator must have $|\alpha^{(i)}|=O(n)$.
\end{lemma}

\begin{proof}
Let $i$ be such that $|\alpha^{(i)}|=\max_{j}|\alpha^{(j)}|$. We have 
\begin{align*}
    \|w\|^{2}&=\sum_{i=1}^{n}(\alpha^{(i)})^{2}\|x^{(i)}\|^{2}+\sum_{i\neq j}\alpha^{(i)}\alpha^{(j)}x^{(i)}\cdot x^{(j)} \\
    &\geq|\alpha^{(i)}|^{2}(1-O(\frac{1}{n^{3}}))-n^{2}|\alpha^{(i)}|^{2}O(\frac{1}{n^{6}}) \\
    &=|\alpha^{(i)}|^{2}(1-O(\frac{1}{n^{3}})).
\end{align*}
If $|\alpha^{(i)}|=\Omega(n)$, then $\|w\|^{2}=\Omega(n^{2})$, but we know that a minimum norm separator has $\|w\|^{2}=O(n)$ by the previous lemma. This completes the proof.
\end{proof}

\begin{lemma} \label{thm: ws wc O(1)}
Any minimum norm separator $w$ has $|w_c| = O(1)$ and $|w_s| = O(1)$ with probability at least $1-4/2000$.
\end{lemma}
\begin{proof}
Let $r_1 = \mathbb{P}(yx_c \leq -1/2)$ and $r_2 = \mathbb{P}(yx_c \geq 1/2)$, and note that $r_1, r_2 > 0$ are constants independent of $n$. By Hoeffding's inequality, there exists $n_0$ (which can depend on $p, r_1, r_2$) such that for all $n \geq n_0$, with probability at least $1-1/2000$, all four of the following conditions hold simultaneously:
\begin{enumerate}
    \item At least $\frac12 r_1 \nmaj$ of the majority points have $y^{(i)}x_c^{(i)} \leq -1/2$. (This will be used for the case $w_s \leq 0$, $w_c \geq 0$.)
    \item At least $\frac12 r_2 \nmaj$ of the majority points have $y^{(i)}x_c^{(i)} \geq 1/2$. ($w_s \leq 0$, $w_c \leq 0$)
    \item At least $\frac12 r_1 \nmin$ of the minority points have $y^{(i)}x_c^{(i)} \leq -1/2$. ($w_s \geq 0$, $w_c \geq 0$)
    \item At least $\frac12 r_2 \nmin$ of the minority points have $y^{(i)}x_c^{(i)} \geq 1/2$. ($w_s \geq 0$, $w_c \leq 0$)
\end{enumerate}
We will show that $|w_c| = O(1)$ in the first case; the remaining three cases hold via nearly identical arguments. Suppose that we are in the first case, i.e., $w$ has $w_s \leq 0$ and $w_c \geq 0$. Then observe that
\begin{align}
    \frac{\|w\|^2}{n} &\geq \frac{1}{n}\sum_{\substack{i \in \mathcal{I}_{\textrm{maj}} \\ y^{(i)}x_c^{(i)} \leq -1/2}} (1 - w_s - y^{(i)}x_c^{(i)}w_c)_+^2 - O\left(\frac{1}{n^3}\right) \label{eq: w_c bdd} \\
    &\geq \frac{1}{n} \frac12 r_1 \nmaj (1 + \frac12 w_c)^2 - O\left(\frac{1}{n^3}\right) \\
    &= \frac{r_1 \pmaj}{2}(1 + \frac12 w_c)^2 - O\left(\frac{1}{n^3}\right).
\end{align}
Note that this final expression goes to infinity at a rate independent of $n$ as $w_c \geq 0$ increases. Since we know that $\|w\|^2/n = O(1)$ with probability at least $1-1/2000$ (this is the case when we do not use $w_c$ or $w_s$ and simply memorize all the points), a minimum norm separator must have $\|w\|^2/n = O(1)$ as well. In particular, this means that $w_c \geq 0$ must remain bounded independent of $n$. To complete the remaining cases, follow the same logic, but replace the indices of summation in \eqref{eq: w_c bdd} with $i \in \mathcal{I}_{\textrm{maj}}, y^{(i)}x_c^{(i)} \geq 1/2$ for case 2, and so on for cases 3 and 4. Taking a union bound of the failure probabilities completes the proof for $w_c$ with a failure probability at most $2/2000$. The same argument (actually it is simpler because there is no noise in the spurious feature) shows the result for $w_s$.
\end{proof}

\begin{lemma} \label{thm: a^i error}
A minimum norm separator has $|\alpha^{(i)}-y^{(i)}(1-w_{s}-w_{c}x_{c}^{(i)})_{+}|=O(\frac{1}{n^{2}})$
for $i\in\mathcal{{I}_{\textrm{{maj}}}}$ and $|\alpha^{(i)}-y^{(i)}(\lambda+w_{s}-w_{c}x_{c}^{(i)})_{+}|=O(\frac{1}{n^{2}})$
for $i\in\mathcal{{I}_{\textrm{{min}}}}$.
\end{lemma}

\begin{proof}
Assume that the high probability events of Lemmas~\ref{thm: sagawa lemmas} and \ref{thm: ws wc O(1)} hold; this happens with probability at least $1-5/2000$. Observe that $\|w\|^{2}$ is increasing in $|\alpha^{(i)}|$ as long as $|\alpha^{(i)}| = \Omega(\frac{1}{n^4})$. We have
\begin{align*}
    \|w\|^2 &= (\alpha^{(i)})^2 \|x_n^{(i)}\|^2 + \sum_{j\neq i} \alpha^{(i)} \alpha^{(j)} x_n^{(i)} \cdot x_n^{(j)} + (\textrm{constant terms in } \alpha^{(i)}).
\end{align*}
This is a quadratic in $\alpha^{(i)}$. The coefficient on $(\alpha^{(i)})^2$ is $\|x_n^{(i)}\|^2 \geq 1 - O(\frac{1}{n^3})$, and the absolute value of the coefficient on $\alpha^{(i)}$ is $|\sum_{j\neq i} \alpha^{(j)} x^{(i)}\cdot x^{(j)}| \leq n \cdot O(n) \cdot O(\frac{1}{n^6}) = O(\frac{1}{n^4})$. Thus $\|w\|^2$ is increasing for $|\alpha^{(i)}| \geq O(\frac{1}{n^4})/(1-O(\frac{1}{n^3})) = O(\frac{1}{n^4})$, so we can choose $c_0 = O(\frac{1}{n^4})$ such that $\|w\|^2$ is increasing in $|\alpha^{(i)}|$ for $|\alpha^{(i)}|\geq c_0$.

Next, we examine the margin constraints on a separator. We will just examine the case that $i\in \mathcal{I}_{\text{maj}}$ and $y^{(i)} = 1$; the other cases are nearly identical. For such a point, we have
\begin{equation*}
w\cdot x^{(i)} = w_s + w_c x_c^{(i)} + \alpha^{(i)} \|x_n^{(i)}\|^2 + \sum_{j \neq i} \alpha^{(j)} x_n^{(i)} \cdot x_n^{(j)} \geq 1 \hspace{.1in} \end{equation*}
From this, it follows that the $i$-th point satisfies the margin constraint iff
\begin{equation}
\alpha^{(i)} \geq \frac{1}{\|x_n^{(i)}\|^2}(1 - w_s - w_c x_c^{(i)} - \sum_{i\neq j} \alpha^{(j)} x_n^{(i)} \cdot x_n^{(j)}). \label{eq: a^i 0}
\end{equation}
The lower bound \eqref{eq: a^i 0} is unwieldy because it depends on the other $a^{(j)}$, but by Lemmas~\ref{thm: sagawa lemmas} and \ref{thm: |a^i| bound}, \eqref{eq: a^i 0} admits a precise form up to an $O(n^{-2})$ correction. Observe that for the RHS of \eqref{eq: a^i 0}, we have:
\begin{align}
\eqref{eq: a^i 0} &\geq (1-O(1/n^3))(1 - w_s - w_cx_c^{(i)} - O(\frac{1}{n^4})) \nn \\[5pt]
&\geq 1 - w_s - w_cx_c^{(i)} - O(\frac{1}{n^4}) - O(\frac{\log n}{n^3}) \label{eq: a^i 1} \\[5pt]
&\geq 1 - w_s - w_c x_c^{(i)} - O(\frac{1}{n^2}). \label{eq: a^i 2}
\end{align}
Here \eqref{eq: a^i 1} follows because $w_s, w_c = O(1)$ and $ |x_c^{(i)}| = O(\log n)$ with probability at least $1-1/2000$. Similarly, we have
\begin{align}
\eqref{eq: a^i 0} &\leq (1+O(1/n^3))(1 - w_s - w_cx_c^{(i)} + O(\frac{1}{n^4})) \nn \\[5pt]
&\leq 1 - w_s - w_cx_c^{(i)} + O(\frac{1}{n^4}) + O(\frac{\log n}{n^3}) \nn \\[5pt]
&\geq 1 - w_s - w_c x_c^{(i)} + O(\frac{1}{n^2}). \label{eq: a^i 3}
\end{align}
Thus we can let $c_1 = O(n^{-2})$ be chosen so that $$ \left| \frac{1}{\|x_n^{(i)}\|^2}\left(1-w_s-w_cx_c^{(i)} - \sum_{i\neq j} \alpha^{(j)} x_n^{(i)}\cdot x_n^{(j)}\right) - (1-w_s-w_cx_c^{(i)})\right| \leq c_1. $$

Now we know that $a^{(i)}$ will be chosen according to two criteria: (i) subject to the constraint \eqref{eq: a^i 0}, and (ii) to minimize $\|w\|^2$. From these two criteria, the definition of $c_0$, and the lower and upper bounds on \eqref{eq: a^i 0}, we conclude that $\alpha^{(i)}$ must be a number between $\max\{-c_0, \, 1-w_s-w_cx_c-c_1\}$ and $\max\{c_0, \, 1-w_s-w_cx_c+c_1\}$. A simple casework argument shows that the endpoints of this interval are always within $O(n^{-2})$ distance from $(1-w_s-w_cx_c)_+$. Taking a union bound over the failure probabilities from Lemmas~\ref{thm: sagawa lemmas}, \ref{thm: ws wc O(1)}, and $|x_c^{(i)}| = O(\log n)$ shows that this result holds with probability at least $1-6/2000$, completing the proof.
\end{proof}

In the remainder of the proofs, we will define $f(w) = w_s^2 + w_c^2 + \sum_{i\in\mathcal{I}_{\textrm{maj}}} (1-w_s-w_c x_c^{(i)})_+^2 + \sum_{i\in\mathcal{I}_{\textrm{min}}} (\lambda+w_s-w_c x_c^{(i)})_+^2$, so that
$$ \mathbb{E}[f(w)] = w_s^2+w_c^2 + \nmaj \mathbb{E}[(1 - w_s - w_c + w_c z)_+^2] + \nmin \mathbb{E}[(\lambda + w_s - w_c + w_c z)_+^2]. $$
This is the expected squared norm of $w$ treating $w_s$, $w_c$, and the $\alpha^{(i)}$ as parameters and in expectation over the randomness in the $x^{(i)}$, under the assumption that the $x^{(i)}_n$ are \emph{perfectly orthonormal}. A combination of Lemma~\ref{thm: a^i error} and a Bernstein bound will show that $\|w\|^2$ concentrates tightly around $\mathbb{E}[f(w)]$ with high probability, which we now prove.

\begin{lemma} \label{thm: bernstein}
With probability at least $0.99$, we have $|\|w\|^2 - \mathbb{E}[f(w)]| = O(\sqrt{n})$.
\end{lemma}

\begin{proof}
From Lemma~\ref{thm: a^i error}, we know that
$$ \alpha^{(i)} = \begin{cases} y^{(i)}(1 - w_s - w_c x_c^{(i)})_+\pm O(\frac{1}{n^2}) & i \in G_{\textrm{maj}} \\ y^{(i)}(\lambda + w_s - w_c x^{(i)}_c)_+\pm O(\frac{1}{n^2}) & i \in G_{\textrm{min}} \end{cases} $$
with probability at least $1-6/2000$. It follows immediately that $(\alpha^{(i)})^2\|x_n^{(i)}\|^2 = (1-w_s-w_c x_c^{(i)})_+^2 \pm O(\frac{\log n}{n^2})$ for majority points and similarly for the minority points. (Here we have used that $|x_c^{(i)}| = O(\log n)$ with probability at least $1-1/2000$, and that $\|x_n^{(i)}\|^2 = 1 \pm O(n^{-3})$.)  From this fact and Lemma~\ref{thm: sagawa lemmas} (which holds with probability at least $1-1/2000$), we have
\begin{equation} \label{eq: bernstein 1}
|\|w\|^2 - f(w)| \leq n \cdot O(\frac{\log n}{n^2}) + \left|\sum_{i\neq j} \alpha^{(i)} \alpha^{(j)} x_n^{(i)} \cdot x_n^{(j)} \right| = O(\frac{\log n}{n}).
\end{equation}
We now show that $f(w)$ will be close to its expectation with high probability. Observe that $(1-w_s-w_c x_c^{(i)})_+^2$ and $(\lambda-w_s-w_c x_c^{(i)})_+^2$ are all $(2(\lambda +|w_s|+|w_c|)+ |w_c|)^2 = O(1)$ sub-exponential. This follows from the simple fact that $\|c\|_{\psi_2} \leq 2c$ for any constant $c$, $(\cdot)_+$ is 1-Lipschitz and therefore does not increase the sub-Gaussian norm, and from the fact that $\|Z^2\|_{\psi_1} = \|Z\|_{\psi_2}^2$ for any sub-Gaussian random variable $Z$. Thus by Bernstein's inequality, there exists a constant $c_2 = O(1)$ (depending on $\lambda, w_s, w_c$) such that 
\begin{equation} \label{eq: bernstein 2}
 \mathbb{P}(|f(w)-\mathbb{E}f(w)| \geq t) \leq 2\exp\left( -\min\left(\frac{t^2}{nc_2^2}, \frac{t}{c_2}\right)\right).
\end{equation}
Letting $t_0 = c_2\sqrt{n \log 2000} = O(\sqrt{n})$, we see that $|f(w)-\mathbb{E}f(w)| \leq 2t_0 = O(\sqrt{n})$ with probability at least $1-1/2000$ for all sufficiently large $n$. Finally, using the triangle inequality $|\|w\|^2 - \mathbb{E}f(w)|\leq |\|w\|^2 - f(w)| + |f(w)-\mathbb{E}f(w)|$ and substituting the bounds \eqref{eq: bernstein 1} and \eqref{eq: bernstein 2} on these two terms yields the desired result. Taking a union bound over the failure probabilities shows that this fails with probability at most $9/2000 < 1/100$.
\end{proof}

Lemma~\ref{thm: bernstein} shows us that $\frac{\|w\|^2}{n} = \frac{\mathbb{E}f(w)}{n} + O(n^{-1/2})$ with high probability, so it suffices to prove Theorem~\ref{theorem:betterthanrandom} for $\mathbb{E}f(w)$. Note that by the construction of the data generating distribution, $w_c-w_s>0$ means that the classifier has better than random accuracy on the majority group, and $w_c+w_s>0$ means that the classifier has better than random accuracy on the minority group. The remainder of the proof will therefore be spent analyzing $\mathbb{E}f(w)$ and showing that $w_c-w_s, w_c+w_s > 0$ for a minimum norm separator and for a specific range of values of $\lambda$. 

\begin{lemma} \label{thm: wc pos}
For any $w$, replacing $w_c$ with $|w_c|$ does not increase \eqref{eq: separator norm}. Thus, we may assume WLOG that $w_c \geq 0$.
\end{lemma}

\begin{proof}
If $w_c \geq 0$ the statement is obvious, so assume that $w_c < 0$. Since $z \sim N(0,1)$ in \eqref{eq: separator norm}, it suffices to show that
$$ (1 - w_s - w_c + w_c z)_+ \geq (1 - w_s + w_c + w_c z)_+ $$
for any $z$. The above inequality holds because $(\cdot)_+$ is nondecreasing and $w_c < 0$, so we are done.
\end{proof}

\blockcomment{
\begin{lemma} \label{thm: separator lb}
For any $w_s$, we have $$\mathbb{E}[\|w\|^2] \geq (0.46) \cdot \begin{cases}\nmaj(1-w_s)^2 & w_s < -\lambda \\ (\nmaj(1-w_s)^2 + \nmin(\lambda + s)^2) & -\lambda \leq w_s \leq 1 \\ \nmin(\lambda + w_s)^2 & w_s > 1\end{cases}.$$
\end{lemma}

\begin{proof}
Let us analyze the two expectation terms in  \eqref{eq: separator norm}. By Lemma~\ref{thm: wc pos}, we may assume that $w_c \geq 0$. If $1-w_s > 0$, then we can write $w_c = \alpha(1-w_s)$ for some $\alpha \geq 0$. Combining this with \eqref{eq: gauss integral} yields
\begin{align}
    \mathbb{E}[(1-w_s-w_c+w_c)_+^2] &= ((1-w_s-w_c)^2 + w_c^2)\Phi\left(\frac{1-w_s}{w_c} - 1\right) + \frac{(1-w_s-w_c)w_c}{\sqrt{2\pi}} e^{-\frac{(1-w_s-w_c)^2}{2w_c^2}} \nn \\[5pt]
    &= (1-w_s)^2 \left[((1-\alpha)^2 + \alpha^2)\Phi\left(\frac1\alpha - 1\right) + \frac{(1-\alpha)\alpha}{\sqrt{2\pi}}e^{-\frac{(1-\alpha)^2}{2\alpha^2}} \right]. \label{eq: nmaj term}
\end{align}
The RHS of this equation can be minimized numerically as a function of $\alpha$. For $\alpha \geq 0$, it reaches a minimum greater than $0.46$ at $\alpha \approx 0.675$. We can repeat this procedure for the other expectation. Assuming $\lambda + w_s > 0$, we can write $w_c = \beta (\lambda + w_s)$, and we have
\begin{align}
    \mathbb{E}[(\lambda+w_s-w_c+w_c)_+^2] &= ((\lambda+w_s-w_c)^2 + w_c^2)\Phi\left(\frac{\lambda+w_s}{w_c} - 1\right) + \frac{(\lambda+w_s-w_c)w_c}{\sqrt{2\pi}} e^{-\frac{(\lambda+w_s-w_c)^2}{2w_c^2}} \nn \\[5pt]
    &= (\lambda+w_s)^2 \left[((1-\beta)^2 + \beta^2)\Phi\left(\frac1\beta - 1\right) + \frac{(1-\beta)\beta}{\sqrt{2\pi}}e^{-\frac{(1-\beta)^2}{2\beta^2}} \right] \label{eq: nmin term}.
\end{align}
This has precisely the same form as before, and the function of $\beta$ in brackets obtains a minimum greater than $0.46$. Substituting these inequalities into \eqref{eq: separator norm} and dropping the $w_s^2$ and $w_c^2$ terms, we obtain the desired inequality for $-\lambda \leq w_s \leq 1$.

It remains to lower bound the cases $w_s < -\lambda$ and $w_s > 1$. Note that if $w_s < -\lambda$, then clearly $1-w_s > 0$ and the lower bound from \eqref{eq: nmaj term} applies. Combining this inequality with the trivial lower bound $\mathbb{E}[(\lambda + w_s - w_c + w_c z)_+^2] \geq 0$ yields the desired inequality when $w_s < -\lambda$. Applying similar logic and the lower bound from \eqref{eq: nmin term} gives the result for $w_s > 1$.
\end{proof}

\begin{theorem}
Let $p = \pmaj$ for notational convenience. For any choice of $\lambda \geq 0$, we have
$$ 1 - 1.48\sqrt{\frac12 + \frac{1-p}{p}\left(\lambda^2 - 2\lambda + \frac94\right)} \leq w_s \leq -\lambda + 1.48\sqrt{\lambda^2 - 2\lambda + \frac94 + \frac{p}{2(1-p)}}. $$
\end{theorem}

\begin{proof}
From Lemma~\ref{thm: separator lb}, we have that 
$$\frac1n \mathbb{E}[\|w\|^2] \geq \begin{cases} 0.46p(1-w_s)^2 & w_s \leq 0 \\ 0.46(1-p)(\lambda + w_s) & w_s > 0 \end{cases}.$$
For any minimum norm separator $w$, we also have that
$$\frac1n \mathbb{E}[\|w\|^2] \leq \frac1n \mathbb{E}[\|\spu\|^2] \leq \frac{p}{2} + (1-p)(\lambda^2 - 2\lambda + \frac94).$$
Here the second inequality follows from \eqref{eq: use-core 1}. The theorem results by solving the quadratic inequalities
$$ 0.46p(1-w_s)^2 \leq \frac{p}{2} + (1-p)(\lambda^2 - 2\lambda + \frac94), \hspace{.15in} w_s \leq 0 $$
and
$$ 0.46(1-p)(\lambda + w_s)^2 \leq \frac{p}{2} + (1-p)(\lambda^2 - 2\lambda + \frac94), \hspace{.15in} w_s > 0. $$
\end{proof}

\paragraph{Remark.} The resulting interval for $w_s$ is somewhat trivial. In particular, the upper bound is always at least on the order of $\sqrt{\frac{p}{1-p}}$, and as $p\rightarrow 1$ this is very bad. I think I may need to improve the upper bound in \eqref{eq: use-core 1}. This approach also does not say anything about the relative size of $w_s$ and $w_c$
}

\begin{lemma} \label{thm: u v rewrite}
Let $u = w_c + w_s$ and $v = w_c - w_s$. Then we have
\begin{align*}
\mathbb{E}[\|w\|^2] &\geq \nmaj \left[ \left((1-u)^2 + \left(\frac{u+v}{2}\right)^2\right)\Phi\left(\frac{2(1-u)}{u+v}\right) + \frac{(1-u)(u+v)}{2\sqrt{2\pi}}e^{-\frac{2(1-u)^2}{(u+v)^2}} \right] \bigg\}\mathrm{(I)} \\[5pt]
&+ \nmin \left[ \left((\lambda-v)^2 + \left(\frac{u+v}{2}\right)^2\right)\Phi\left(\frac{2(\lambda-v)}{u+v}\right) + \frac{(\lambda-v)(u+v)}{2\sqrt{2\pi}}e^{-\frac{2(\lambda-v)^2}{(u+v)^2}} \right] \bigg\}\mathrm{(II)}
\end{align*}
\end{lemma}

\begin{proof}
We have $w_c = (u+v)/2$. The inequality follows by dropping the $w_c^2 + w_s^2$ terms in \eqref{eq: separator norm}, rewriting the remaining terms in terms of $u$ and $v$, and then applying formula \eqref{eq: gauss integral} to each of the expectations.
\end{proof}

\begin{theorem}[Formal version of Theorem \ref{theorem:betterthanrandom}]
For $\frac{p}{8(1-p)} + \frac14(1-\frac{1}{\sqrt{2\pi e}}) \leq \lambda \leq \frac{1+1/\sqrt{2}}{8(1-p)}$, with probability at least $0.99$, we have that $u, v > 0$. In particular, for this range of margins $\lambda$, IT has strictly better than random worst-group accuracy.
\end{theorem}

\begin{proof}
By Lemma~\ref{thm: wc pos}, we may assume that $w_c = (u+v)/2 \geq 0$. If $1-u \geq 0$, then $1-u = \alpha(u+v)/2$ for some $\alpha$. Then (I) becomes
\begin{align*}
    \mathrm{(I)} &= w_c^2\left[(\alpha^2+1)\Phi(\alpha)+\frac{\alpha}{2}e^{-\alpha^2/2}\right].
\end{align*}
If we similarly let $\lambda - v = \beta(u+v)/2$, then we see that
\begin{equation*}
\text{(II)} = w_c^2 \left[(\beta^2 + 1)\Phi(\beta) + \frac\beta2 e^{-\beta^2/2}\right].
\end{equation*}

Combining these, we have:
\begin{equation} \label{eq: main sep lb}
    \frac{\mathbb{E}[\|w\|^2]}{n} \geq pw_c^2\left[(\alpha^2 + 1)\Phi(\alpha) + \frac{\alpha}{\sqrt{2\pi}}e^{-\alpha^2/2} \right] + (1-p)w_c^2 \left[ (\beta^2 + 1)\Phi(\beta) + \frac{\beta}{\sqrt{2\pi}}e^{-\beta^2/2} \right]
\end{equation}
with the identities $\alpha w_c = 1-u$, $\beta w_c = \lambda - v$, and $w_c = (u+v)/2$.

Before we proceed, let us first examine the quantity $f(x) = (x^2 + 1)\Phi(x) + \frac{x}{\sqrt{2\pi}}e^{-x^2/2}$. Observe that for $x \geq 0$, we have $f(x) \geq \frac12(x^2+1)$. Furthermore, if $x\geq 1$, we have the tighter lower bound
\begin{equation} \label{eq: f lb}
\Phi(x) \geq 1 - \frac{1}{\sqrt{2\pi}}e^{-x^2/2} \hspace{.1in} \Longrightarrow \hspace{.1in} f(x) \geq x^2 + 1 - \frac{1}{\sqrt{2\pi}}e^{-x^2/2}(x^2-x+1) \geq x^2 + 1 - \frac{1}{\sqrt{2\pi e}}. 
\end{equation}
The lower bound on $\Phi(x)$ comes from the standard Gaussian tail bound $\mathbb{P}(Z \geq x) \leq \frac{1}{\sqrt{2\pi}}e^{-x^2/2}$ for $x\geq 1$. The second inequality on the RHS of \eqref{eq: f lb} follows from maximizing $e^{-x^2/2}(x^2-x+1)$ over $x \geq 1$.

Recalling that (I) and (II) were defined as expectations of nonnegative quantities, we also have the lower bound $f(x) \geq 0$ for $x < 0$. We can now use \eqref{eq: main sep lb} to show that whenever $u \leq 0$ or $v \leq 0$, we must have $\mathbb{E}[\|w\|^2] > \mathbb{E}[\|\core\|^2]$ for an appropriate choice of $\lambda$.

Note that since $w_c = (u+v)/2$ and we know that $w_c \geq 0$ for any separator, at most one of $u,v$ may be negative. Thus, the four cases that follow cover all possibilities. Also, recall that $\alpha$ and $\beta$ are defined such that $\alpha = \frac{2(1-u)}{u+v}$ and $\beta = \frac{2(\lambda - v)}{u+v}$.

\paragraph{Case 1.1: $u \leq 0$ and $v > \lambda$.}
In this case, $\alpha \geq 0$, and we can replace \eqref{eq: main sep lb} with
\begin{equation} \label{eq: lb 1.1}
\frac{\mathbb{E}[\|w\|^2]}{n} \geq \frac12 p w_c^2 (\alpha^2 + 1) = \frac12 p \left((1-u)^2 + \frac{(u+v)^2}{4}\right).
\end{equation}
(This uses the fact that $f(\alpha) \geq \frac12(\alpha^2+1)$ for $\alpha \geq 0$.) The minimum of \eqref{eq: lb 1.1} over $u$ occurs at $u = \frac45 - \frac{v}{5}$ and has value at least $\frac{p}{10}(\lambda+1)^2$ since $v \geq \lambda$. Thus we have $\frac{\mathbb{E}[\|w\|^2]}{n} \geq \frac{p}{10}(\lambda+1)^2$ in this case.

\paragraph{Case 1.2: $u \leq 0$ and $0 \leq v \leq \lambda$.}
In this case, $\alpha, \beta \geq 0$. We further split into two subcases based on whether $\beta \leq 1$ or $\beta > 1$.

If $\beta > 1$, then we can apply \eqref{eq: f lb} to the $\beta$ portion of \eqref{eq: main sep lb} as well as the lower bound for $\alpha \geq 0$ to the other part. This yields
\begin{align} 
\frac{\mathbb{E}[\|w\|^2]}{n} &\geq \frac12 pw_c^2(\alpha^2 + 1) + (1-p)w_c^2(\beta^2+1-\frac{1}{\sqrt{2\pi e}}) \nn \\[5pt]
&= \frac12 p \left((1-u)^2 + \frac{(u+v)^2}{4}\right) + (1-p)\left((\lambda-v)^2 + c_1(u+v)^2\right), \label{eq: lb 1.2.1}
\end{align}
where $c_1 = \frac14 (1-\frac{1}{\sqrt{2\pi e}})$. The minimum of \eqref{eq: lb 1.2.1} over $v$ occurs at
$$ v = \frac{2(1-p)\lambda - \left(2c_1(1-p)+\frac{p}{4}\right)u}{2(1+c)(1-p)+\frac{p}{4}}. $$
Substituting this into \eqref{eq: lb 1.2.1}, we obtain
\begin{align}
    \frac{\mathbb{E}[\|w\|^2]}{n} &\geq \frac{-16 c_1 \lambda^2 (-1 + p)^2 - 8 p + 8 c_1 (-1 + p) p + 2 \lambda^2 (-1 + p) p + 7 p^2}{ 2 (-8 + 8 c_1 (-1 + p) + 7 p)} \label{eq: lb 1.2.1 substitute u} \\[5pt]
    &+\frac{-32 c_1 \lambda (-1 + p)^2 + 16 p - 
    16 c_1 (-1 + p) p + 4 \lambda (-1 + p) p - 14 p^2 }{ 2 (-8 + 8 c_1 (-1 + p) + 7 p)} u \nn \\[5pt]
    &+\frac{16 c_1 (-1 + p) - 10 p - 
    8 c_1 (-1 + p) p + 9 p^2}{2 (-8 + 8 c_1 (-1 + p) + 7 p} u^2. \nn
\end{align}
Taking the derivative of \eqref{eq: lb 1.2.1 substitute u} with respect to $u$, we arrive at
$$ \frac{32c_1\lambda(1-p)^2 - 16p - 16c_1p(1-p)+4\lambda(1-p)p+14p^2}{2(8+8c_1(1-p)-7p)} + \frac{16(1-p)+10p-8c_1p(1-p)-9p^2}{8+8c_1(1-p)-7p}u. $$
Observe that since $0 < c_1, p < 1$, the coefficient on $u$ is nonnegative. If $\lambda \leq \frac{1}{2(1-p)}$, then the constant term in this expression is also negative. Thus for $u \leq 0$, \eqref{eq: lb 1.2.1 substitute u} is decreasing in $u$ and the minimum occurs at $u=0$. Substituting $u=0$, we finally find that
$$ \frac{\mathbb{E}[\|w\|^2]}{n} \geq \underbrace{\frac{16c_1(1-p)+2p}{2(8+8c_1(1-p)-7p)}}_{c_2}(1-p)\lambda^2 + \underbrace{\frac{8p+8c_1p(1-p)-7p^2}{2(8+8c_1(1-p)-7p)}}_{c_3}. $$

Otherwise, $\beta \leq 1$. In this case, we have
$$ \beta = \frac{2(\lambda - v)}{u+v} \leq 1 \hspace{.1in} \Longrightarrow \hspace{.1in} v \geq \frac{2\lambda - u}{3} \geq \frac23\lambda $$
since $u \leq 0$. Since both $\alpha, \beta \geq 0$, \eqref{eq: main sep lb} can be lower bounded by
\begin{equation} \label{eq: lb 1.2.2}
\frac{\mathbb{E}[\|w\|^2]}{n} \geq \frac12 pw_c^2(\alpha^2 + 1) + \frac12 (1-p)w_c^2(\beta^2+1) = \frac12 \left(p(1-u)^2 + \frac{(u+v)^2}{4} + (1-p)(\lambda-v)^2 \right).
\end{equation}
The minimum of \eqref{eq: lb 1.2.2} over $u$ occurs at $u = \frac{4p-v}{4p+1}$, at which point we have
\begin{equation} \label{eq: lb 1.2.2'}
\frac{\mathbb{E}[\|w\|^2]}{n} \geq \frac12 \left( p\left(1-\frac{4p-v}{1+4p}\right)^2 + (1-p)(\lambda-v)^2 + \frac14\left( \frac{4p-v}{1+4p}+v\right)^2\right).
\end{equation}
The derivative of the above with respect to $v$ is
$$ \frac{p - (1 + 3p - 4p^4)\lambda + (1 + 4p - 4p^2)v}{1+4p}, $$
which is positive when $v \geq \frac{(1+3p-4p^2)\lambda - p}{1+4p-4p^2}$; in particular, it is positive for $v \geq \frac23 \lambda$, and therefore plugging $v=\frac23 \lambda$ into \eqref{eq: lb 1.2.2'} gives us the lower bound
$$ \frac{\mathbb{E}[\|w\|^2]}{n} \geq \frac{(1+7p-4p^2)\lambda^2 + 12p\lambda + 9p}{18 + 72p}. $$
For $p \approx 1$, this lower bound is greater than the one we obtained for the $\beta > 1$ case. Thus we can conclude that
$$ \frac{\mathbb{E}[\|w\|^2]}{n} \geq c_2(1-p)\lambda^2 + c_3$$
whenever $u\leq 0$ and $0\leq v \leq \lambda$.

\paragraph{Case 2.1: $v \leq 0$ and $u > 1$.}
In this case, $\beta \geq 0$. We further split into two sub-cases depending on the size of $\beta$.

If $0\leq \beta \leq 1$, then note that
$$\beta = \frac{2(\lambda - v)}{u+v} \leq 1 \hspace{.1in} \Longrightarrow \hspace{.1in} 0\geq v \geq \frac{2\lambda - u}{3}.$$
Thus we must have $u \geq 2\lambda$. With this in mind, we can replace \eqref{eq: main sep lb} with
\begin{equation} \label{eq: lb 2.1}
\frac{\mathbb{E}[\|w\|^2]}{n} \geq \frac12 (1-p)\left( (\lambda-v)^2 + \frac{(u+v)^2}{4} \right).
\end{equation}
The minimum of \eqref{eq: lb 2.1} over $v \leq 0$ occurs at $v=0$ if $u\leq 4\lambda$ and $v = \frac{4\lambda - u}{5})$ if $u > 4\lambda$. In the former case, we have
$$\frac{\mathbb{E}[\|w\|^2]}{n} \geq \frac12(1-p)(\lambda^2 + \frac{u^2}{4}) \geq (1-p)\lambda^2$$
since $u \geq 2\lambda$. In the latter case, substituting $v = \frac{4\lambda - u}{5}$ into \eqref{eq: lb 2.1} and recalling that $u > 4\lambda$, we have
$$\frac{\mathbb{E}[\|w\|^2]}{n} \geq \frac{1}{10}(1-p)(\lambda + u)^2 \geq \frac52 (1-p)\lambda^2.$$
We always have the lesser of these two lower bounds, namely $\frac{\mathbb{E}[\|w\|^2]}{n} \geq (1-p)\lambda^2$.

Otherwise, we have $\beta > 1$ and we can apply inequality \eqref{eq: f lb} to \eqref{eq: main sep lb}. This yields
\begin{equation} \label{eq: lb 2.1.2}
\frac{\mathbb{E}[\|w\|^2]}{n} \geq (1-p)w_c^2 \left(\beta^2 + 1 - \frac{1}{\sqrt{2\pi e}}\right) = (1-p)\left((\lambda-v)^2 + c_1(u+v)^2\right),
\end{equation}
where $c_1 = \frac14(1-\frac{1}{\sqrt{2\pi e}})$. We can minimize the above expression with respect to $v \leq 0$. The minimum occurs at $v=0$ when $u\leq \frac{\lambda}{c_1}$ and at $v = \frac{\lambda - c_1u}{1+c_1}$ when $u > \frac{\lambda}{c_1}$. In the first case, we have
$$ \frac{\mathbb{E}[\|w\|^2]}{n} \geq (1-p)(\lambda^2 + c_1u^2) \geq (1-p)(\lambda^2 + c_1)$$
since $u > 1$. In the second case, we have
\begin{align*}
    \frac{\mathbb{E}[\|w\|^2]}{n} &\geq (1-p)\left( \left(\frac{(1+c_1)\lambda - \lambda + c_1u}{1+c_1}\right)^2 + c_1\left(\frac{(1+c_1)u+\lambda-c_1u}{1+c_1}\right)^2\right) \\[5pt]
    &= (1-p)\left( \left(\frac{c_1(\lambda + u)}{1+c_1}\right)^2 + c_1\left(\frac{\lambda+u}{1+c_1}\right)^2\right) \\[5pt]
    &\geq (1-p)(1+\frac{1}{c_1})\lambda^2.
\end{align*}
To finish Case 2.1, we always have at least the minimum of the lower bounds which we have obtained in this section, namely $\frac{\mathbb{E}[\|w\|^2]}{n} \geq (1-p)\lambda^2$.

\paragraph{Case 2.2: $v \leq 0$ and $0 \leq u \leq 1$.}
In this case, $\beta, \alpha \geq 0$. In fact, we have the stricter constraint $\beta > 1$. To see this, recall that we showed in Case 2.1 that when $\beta \leq 1$, we have $v \geq \frac{2\lambda - u}{3} > 0$ for $0\leq u \leq 1$. Since we have assumed $v\leq 0$, this cannot happen, thus $\beta > 1$ and we can apply inequality \eqref{eq: f lb} to \eqref{eq: main sep lb}. Since $\alpha \geq 0$ in this setting as well, we have
\begin{equation} \label{eq: lb 2.2}
\frac{\mathbb{E}[\|w\|^2]}{n} \geq \frac12 p\left((1-u)^2 + \frac{(u+v)^2}{4}\right) + (1-p)\left( (\lambda-v)^2 + c_1(u+v)^2\right),
\end{equation}
with $c_1 = \frac14 (1-\frac{1}{\sqrt{2\pi e}})$ as before. We again compute the derivative of \eqref{eq: lb 2.2} with respect to $v$:
\begin{align*}
    p\frac{u+v}{4} + 2(1-p)(v-\lambda + c_1(u+v)) &= \left(\frac{p}{4} + 2(1-p)(1+c_1)\right)v + \left(\frac{pu}{4} + 2(1-p)c_1 u - 2(1-p)\lambda\right) \\[5pt]
    &\leq \left(\frac{p}{4} + 2(1-p)(1+c_1)\right)v + \left(\frac{p}{4} + 2(1-p)(c_1-\lambda)\right).
\end{align*}
Note that this is nonpositive for all $v\leq 0$ provided that $\lambda \geq c_1 + \frac{p}{8(1-p)}$. In this case, \eqref{eq: lb 2.2} is minimized at $v=0$ and we have
$$ \frac{\mathbb{E}[\|w\|^2]}{n} \geq \frac12 p\left((1-u)^2 + \frac{u^2}{4}\right) + (1-p)(\lambda^2 + c_1u^2) \geq (1-p)\lambda^2. $$

\paragraph{Combining the cases.}
Let us now gather the constraints on $\lambda$ as well as the lower bounds on $\frac{\mathbb{E}[\|w\|^2]}{n}$. The smallest lower bound comes from Case 1.2, and we have
$$ \frac{\mathbb{E}[\|w\|^2]}{n} \geq \underbrace{\frac{16c_1(1-p)+2p}{2(8+8c_1(1-p)-7p)}}_{c_2}(1-p)\lambda^2 + \underbrace{\frac{8p+8c_1p(1-p)-7p^2}{2(8+8c_1(1-p)-7p)}}_{c_3}. $$
Case 1.2 also required that $\lambda \leq \frac{1}{2(1-p)}$. From Case 2.2, we also have the constraint $\lambda \geq c_1 + \frac{p}{8(1-p)}$. Our problem is therefore reduced to finding $\lambda$ such that
\begin{equation} \label{eq: core vs bad ineq}
    c_2(1-p)\lambda^2 + c_3 > \frac12 p + (1-p)\left(\lambda^2 - 2\lambda + \frac94\right)
\end{equation}
subject to $c_1 + \frac{p}{8(1-p)} \leq \lambda \leq \frac{1}{2(1-p)}$. The inequality \eqref{eq: core vs bad ineq} reduces to
\begin{equation} \label{eq: lambda range}
(1-c_2)\lambda^2 - 2\lambda + \frac94 + \frac{\frac12p - c_3}{1-p} < 0 \hspace{.1in} \Longrightarrow \hspace{.1in} \lambda \in \textrm{Range}\left(\frac{1 \pm \sqrt{1 - (1-c_2)\left(\frac94 + \frac{\frac12p-c_3}{1-p}\right)}}{1-c_2}\right).
\end{equation}
We now analyze $c_2$ and $c_3$, starting with $c_3$. Observe that
\begin{equation*}
    c_3 = \frac{8p+p\overbrace{(8c_1(1-p)-7p)}^{<0}}{2(8+8c_1(1-p)-7p)} \geq \frac{8p + 8c_1(1-p) - 7p}{2(8+8c_1(1-p)-7p)} \geq \frac12.
\end{equation*}
The inequality holds because $p < 1$. Next, we consider $1-c_2$:
\begin{equation*}
1-c_2 = \frac{16 + 16c_1(1-p) - 14p - 16c_1(1-p) - 2p}{2(8+8c_1(1-p)-7p)} = \frac{8(1-p)}{8+8c_1(1-p)-7p} \leq 8(1-p).
\end{equation*}
Plugging these into our range for $\lambda$, we see that
$$ \frac{1+\sqrt{1-(1-c_2)\left(\frac94 + \frac{\frac12p-c_3}{1-p}\right)}}{1-c_2} \geq \frac{1 + \sqrt{1 - 8(1-p)\left(\frac94 - \frac12 \right)}}{8(1-p)} = \frac{1 + \sqrt{1 - 14(1-p)}}{8(1-p)}. $$
Finally, we see that the range $c_1 + \frac{p}{8(1-p)} < \lambda < \frac{1 + \sqrt{\frac12}}{8(1-p)}$ satisfies both \eqref{eq: lambda range} as well as the constraints $c_1 + \frac{p}{8(1-p)} < \lambda < \frac{1}{2(1-p)}$. This completes the proof.
\end{proof}

\section{Experiment Details}
\label{appendix:expdetail}
\subsection{Label Shift}
For numerical experiments under label shift setting, we train a ResNet-32 \cite{he2016deep} on both Fashion MNIST \cite{xiao2017fashion} and CIFAR-10 \cite{krizhevsky2009learning} datasets. In both cases, we train the ResNet-32 model using stochastic gradient descent method with a momentum term of 0.9, a weight decay rate of 2e-4 and a batch size of 128. Each model is trained for 400 epochs, and we use an adaptive learning rate schedule where the initial learning rate is set to be 0.1 and it will be annealed to 1e-3 after 150 epochs and 1e-5 after 250 epochs. 

In particular, we find that when applying importance tempering method for learning extremely imbalanced dataset, the optimization landscape become much more complicated and hard to optimize. To tackle this problem, we will first use a low temperature to train the model for a number of epochs and then apply the high temperature to train for the remaining epochs. For results in Figure \ref{figure:temperature}, we will first set $\gamma=0.2$ to pretrain the model for 100 epochs and then apply the real $\gamma$ for 300 epochs when the true $\gamma$ is greater than $0.2$. For results in Table \ref{table:result}\ref{table:tempposition}, and Figure \ref{figure:collapse}, we also adopt this pretraining techniques and report the optimal results for $\gamma$ ranging from 0.0 to 1.0.

\subsection{Spurious Correlation}

On Waterbirds and CelebA dataset, we use the Pytorch torchvision implementation of the ResNet50 model, starting from pretrained weights. We train the ResNet50 models using stochastic gradient descent with a momentum term of 0.9 and
1`a batch size of 128; the original paper used batch sizes of 128 or 256 depending on the dataset. Following \cite{sagawa2019distributionally}, we use a fixed learning rate instead of the standard adaptive learning rate schedule so that we can compare the difference between ours and previous methods (avoid introducing more hyperparameters). Different from \cite{sagawa2019distributionally}, we train all the model till 500 epochs, so that we can fully explore the feature space which enables us to get benefit from overparameterization. For the standard training, we select apply a 1e-4 weight decay and for strong $\ell_2$ penalty, we use $\lambda = 1.0$ for waterbirds and $\lambda=0.1$ for CelebA. For CelebA, we select temperature $(1/75,1/100,1/100,1)$ for standard training and $(1/100,1/225,1/275,1)$ for regularized models. For Waterbirds, we select temperature $(1/100,1/50,1,1/75)$ for standard training and $(1/20,1/15,1,1/15)$ for regularized models. We use the Pytorch torchvision implementation of the WideResNet50 model as our larger models. For other hyperparameter, we set the same as the previous setting.  For CelebA, we select temperature $(1/75,1/250,1/250,1)$ for standard training and $(1/100,1/225,1/275,1)$ for regularized models. On MultiNLI dataset, we use huggingface pytorch-transformers implementation \cite{wolf2019huggingface} for the Bert (bert-base-uncased) and Bert large (bert-large-uncased) model, starting from pretrained weights.  We use the
default tokenizer and model settings from that implementation, including a fixed linearly-decaying learning rate starting at 0.00002, AdamW optimizer, dropout, batch size of 32 and no weight decay as \cite{sagawa2019distributionally} implements. We select temperature as $(1/150,1/8000,1.300,3,1/80,1)$.

\end{document}